\documentclass[12pt]{article} 

\usepackage{amsmath,amsfonts,amssymb,mathtools,amsthm}
\usepackage{mathrsfs}

\usepackage[all]{xy}
\usepackage{psfrag}

\theoremstyle{plain}

\newtheorem{global-theorem}{Theorem}
\newtheorem{theorem}{Theorem}[section]
\newtheorem{lemma}[theorem]{Lemma}

\newtheorem{corollary}[theorem]{Corollary}

\newtheorem{proposition}[theorem]{Proposition}

\theoremstyle{definition}

\DeclareFontFamily{U}{rsf}{}
\DeclareFontShape{U}{rsf}{m}{n}{
  <5> <6> rsfs5 <7> <8> <9> rsfs7 <10->  rsfs10}{}
\DeclareMathAlphabet{\mathscr}{U}{rsf}{m}{n}

\newcommand{\lesone}[6]{
\xymatrix{     
 0 \ar[r] & {#1} \ar[r]  &  {#2} \ar[r]  &  {#3} 
\ar@{->}`r/10pt[d] `[l] `^dl[dlll]  `^dr/10pt[dll]    [dll] \\
 &  {#4} \ar[r] & {#5} \ar[r] & {#6} \ar[r] & 0 }
}

\newcommand{\zz}{{\mathbb Z}}
\newcommand{\nn}{{\mathbb N}}
\newcommand{\ff}{{\mathbb F}}

\newcommand{\Pp}{{\mathcal P}}

\newcommand{\srT}{{\mathscr T}}
\newcommand{\srS}{{\mathscr S}}

\newcommand{\srC}{{\mathscr C}}

\title{Learning proofs for the classification of nilpotent semigroups}
\author{Carlos Simpson}
\date{} 

\begin{document}
\maketitle

\begin{abstract}

Machine learning is applied to find proofs, with smaller or smallest numbers of nodes,
for the classification of 4-nilpotent semigroups. 

\end{abstract}



\section{Introduction}

We are interested in the classification of finite semigroups. Distler \cite{Distler1,Distler2,Distler3} has provided a list of isomorphism classes for sizes 
$n\leq 10$, but at great computational expense. The question we pose here is whether artificial intelligence, in the form of deep learning, 
can learn to do a classification
proof for these objects. 

To be more precise, we are going to look at the question of whether a process designed to learn 
how to do proofs using neural networks can
learn to do ``better'' proofs, as measured by the number of nodes in the proof tree. 

Let's point out right away that the process will not, in its current state, be useful for improving in practical
terms the computational time
for a classification proof. Even though we are able to find proofs with small numbers of nodes, potentially close to 
the minimum, the training time necessary to do that
is significantly bigger than the gain with respect to a reasonable benchmark process. Therefore, this study should be considered
more for what it says about the general capacity of a deep learning process to learn how to do proofs. 

Unsurprisingly, the motivation for this question is the recent phenomenal success obtained by Alpha Go and Alpha Zero \cite{AlphaGo,AlphaZero} 
at guiding complex strategy games. If we think of a mathematical proof as a strategy problem, then it seems logical to suppose
that the same kind of technology could guide the strategy of a proof. 

In turn, this investigation serves as a convenient and fun way to experiment with some basic Deep Learning programming. 
In recent years there have been an increasing number of studies of the application of machine learning to mathematics, starting from
\cite{HeLandscape,CarifioLandscape} and continuing, to cite just a very few, with
\cite{DHL, GaussManin, Wagner}. The application we propose here has the property that it generates its own training data,
providing a self-contained microcosm in which to test things like architecture of neural networks and sampling and training processes.

The problem at hand is that of determining the list of isomorphism classes of semigroups of a given kind. Recall that a 
semigroup consists of a set $A$ and a binary operation 
$$
A\times A \rightarrow A \;\;\mbox{ denoted }\;\; (x,y)\mapsto x\cdot y
$$
subject only to the axiom that it is associative,  $\forall x,y,z \in A, \; x\cdot (y\cdot z) = (x\cdot y) \cdot z$. 
For us, the set $A$ will be finite, typically having from $7$ to $10$ (or maybe $11$) elements. 

We envision a simple format of the classification proof, where at each step we make a {\em cut}, branching according
to the possible hypotheses for the values of a single multiplication $x_i\cdot y_i$. The basic strategy question is which
location $(x_i,y_i)$ to choose at each stage of the proof. Once the possible cuts have been exhausted then we
have a classification of the multiplication tables. 

We don't look at the process of filtering according to isomorphism classes---for the sizes under examination,
that doesn't pose any problem in principle but it would generate an additional programming task. 
Nonetheless the symmetry will be used by starting with
an initial hypothesis about the possible multiplication operations; the set of these hypotheses is filtered by a sieve under
the symmetric group action. Typically, our proof learning process will then concentrate on a single instance 
denoted $\sigma$ of this collection of possible initial conditions. 

The organization of the paper is to first 
describe the computational and learning setup used to try to learn proofs, then next make comments on the choice of network architecture and
sampling and training processes. Then, we show some graphs of the results of the learning process on specific
proof problems. 

For a small initial condition corresponding to certain semigroups of size $7$, we can find in another way the precise lower
bound for the size of the proof, as is discussed in Section \ref{minimality}.
Our learning framework is able to attain proofs of the minimal size. For larger
cases it isn't practically possible to obtain a proven lower bound so we can only show that progress is attained,
leaving it for speculation as to the question of how close we are getting to the theoretical lower bound.

\subsection{Classification framework: $4$-nilpotent semigroups}

We now discuss the framework in somewhat greater detail. We are going to be classifying {\em nilpotent semigroups}.
It is useful to understand the role played by nilpotent 
semigroups in the classification of finite semigroups.

They are representative of the phenomena that lead to large numbers of solutions. In order to understand this,
it is good to look at the $3$-nilpotent case, one of the main constructions of a large number of semigroups. 

Suppose given a filtered nilpotent semigroup $A$ of size $n\geq 3$ with filtration having three steps as follows:
$$
F^1A = A = \{ 0, \ldots , n-1\},\;\;\;\;  F^2 A = \{ 0, 1\}, \;\;\;\;  F^3 A = \{ 0\} .
$$
It means that $x\cdot y \in \{ 0,1\}$, and if $x=0$ or $x= 1$ or $y= 0$ or $y= 1$ then $x\cdot y = 0$. 
Clearly, for any multiplication table satisfying these properties we have $(x\cdot y)\cdot z = 0$ and
$x \cdot (y\cdot z)=0$ for any $x,y,z\in A$. Therefore, any multiplication table satisfying these properties
is automatically associative. To specify the table, we must just specify $x\cdot y \in \{ 0,1\}$ for $x,y\in F^1A - F^2 A = \{ 2,\ldots , n-1\}$. 
There are 
$$
2^{(n-2)^2}
$$
possibilities. For example with $n= 6$ this is $2^{16} = 65536$ possibilities. For $n= 10$ we have $2^{64}$ possibilities.\footnote{By
taking $F^2A = \{ 0,\ldots , (n/2)-1\}$ let's say with $n$ even, we obtain $(n/2)$ to the power of $n^2 / 4$ solutions, 
that is roughly the exponential of $n^2 \log (n) / 4$, and dividing by the symmetric group action doesn't diminish that.}

This example illustrates an important phenomenon, and along the way shows that we can expect the nilpotent cases of all kinds
to occupy a large piece in the general classification. Well-known for some time, this 
was the motivation for Distler's paper \cite{Distler3}. 

In the present study we shall look at the $4$-nilpotent case (as was highlighted to us by D. Alfaya). Namely, assume
given  a filtration $A = F^1A$ and $F^kA=A^k$ with $F^4A = \{ 0\}$. We are furthermore going to assume that the 
``associated-graded dimensions'' $| F^k A - F^{k+1}A|$ are $(a,b,1,1)$, and that $A$ is its own associated-graded,
meaning that if a product reduces the filtration level by more than expected then it is zero. These conditions may be
seen to preserve the essential part of the classification question.

The terminology ``associated-graded'' comes from the interpretation of semigroups with absorbing element $0$ as
${\mathbb F}_1$-algebras, i.e. algebras in the monoidal category of ${\mathbb F}_1$-modules, those just being 
pointed sets with tensor operation the smash product.

The next step, in order to understand both the proof mechanism and the encoding of data to feed to the machine, is to discuss
{\em multiplexing}. This is a very standard procedure. 

To give a simple example, suppose we want to make a neural network to 
predict traffic on a road. It will depend on the day of the week. If we give as input data a number $d \in \{ 0,\ldots , 6\}$ it 
probably isn't
going to work very well since the numerical size of $d$ has nothing {\em a priori} to do with the amount of traffic. 
A much better solution would be to give as input data a vector $v\in \zz ^7$ with 
$$
v= (v_0,\ldots , v_6), \;\;\;\; v_i \in \{ 0,1\} , \;\;\;\; v_0 + \ldots + v_6 = 1.
$$
The last two conditions mean that there is exactly one value that equals $1$, the rest are $0$. We have transformed our integer data
from a {\em quantity} to a {\em location}. With the location data, the machine could make a different calculation for each day of 
the week and is much more likely to find a good answer. 

In our situation, the analogous point is that we don't want to give the numerical data of the values $x\cdot y$ in the multiplication
table. These are numerical values in $\{ 0,\ldots , n-1\}$ but their ordering is only somewhat canonical 
in the nilpotent case, and in the general non-nilpotent case it might be highly indeterminate. Therefore, we {\em multiplex}
the multiplication table into an $n\times n\times n$ tensor $m[x,y,z]$ with the rule
$$
m[x,y,z] = 1 \Leftrightarrow x\cdot y = z, \;\;\;\; m[x,y,z] = 0 \mbox{ otherwise}.
$$
Recall here that the indices $x,y,z$ take values in $0,\ldots , n-1$. 

We'll call the tensor $m$ the {\em mask}. This representation has some nice properties. The first of them is that it allows us
to encode not only the multiplication table itself, but also whole collections of multiplication tables.
A {\em mask} is an $n\times n \times n$ tensor $m$ with entries denoted $m[x,y,z]$, such that the entries are all either
$0$ or $1$. A mask is transformed into a condition about multiplication tables as follows: 
$$
m[x,y,z]= 0 \;\; \mbox{ means } \;\; x\cdot y \neq z
$$
whereas
$$
m[x,y,z] = 1 \;\; \mbox{ means } \;\; x\cdot y \;\; \mbox{ might be } \;\; z.
$$
Given a mask $m$, a table $t$ (that is, an $n\times n$ array whose entries $t[x,y]$ are in $\{ 0,\ldots , n-1\} $)
is said to {\em satisfy} $m$ if, for all $x,y\in \{ 0,\ldots , n-1\}$ putting $z:= t[x,y]$ yields $m[x,y,z]= 1$. 

Geometrically we may view the mask as a subset of a $3$-dimensional grid (the subset of points where the value is $1$)
and a table, viewed as a section from the $2$-dimensional space of the first two coordinates into the $3$-dimensional
space, has to have values that land in the subset in order to satisfy the mask. 

For a given mask $m$ there is therefore a set of associative multiplication tables $t$ that satisfy $m$. 

We could formulate our classification problem in this way: there is a mask $Q({\rm nil},n).m$ of size $n$ corresponding to the 
$4$-nilpotency conditions. We would like to classify associative tables that satisfy this mask. 

That setup is rather general. For the $4$-nilpotent case with given associated-graded dimensions $(a,b,1,1)$ we will
rather look at a collection of masks for the multiplication operations involving the two sets having $a$ and $b$
elements respectively, as discussed in Section \ref{nilpotent}.

\subsection{Cuts and the proof tree}

Let us keep to the more general setting of the previous subsection in order to view the notion of a classification proof by
cuts. A {\em position} of the proof is just some mask that we'll now call $p$. For this mask, there may be a certain
number of {\em available $(x,y)$ locations}. We'll say that $(x,y)$ is available if the number of $z$ such that
$p(x,y,z)=1$ is $>1$. Notice that if, for any $(x,y)$ that number is $0$ then the mask is impossible to realize, that is
to say there are no tables that satisfy it, and if that number is $1$ for all $(x,y)$ then the classification proof is done
at that position: the mask determines the multiplication table. Therefore, at any active position $p$ in the proof, i.e. a position
where there is still some proving to be done, there must be at least one available location. 

If we fix some available location $(x,y)$, then making the {\em cut at $(x,y)$} generates a collection of new proof positions
$p_1,\ldots , p_k$. 
Namely, with $k$ the number of values $z$ such that $p(x,y,z)=1$, 
we take the mask $p$ but replace the column corresponding to $(x,y)$ by, sucessively, the $k$ different columns with
a single $1$ and the rest $0$'s, that correspond to values of $z$ where $p(x,y,z)=1$. 
The new positions need to be {\em processed} to apply the logical implications of the associativity axiom, potentially adding
new $0$'s (see Section \ref{appendix} for the functions that implement this processing). 

This collection of new positions 
generated by the cut is going to be associated to the new nodes underneath the given node, thus creating incrementally the 
{\em proof tree}. Once a cut has been made at a given node, it acquires the label ``passive''. The nodes below it
are labelled ``active'', except for those that are ``done'' or ``impossible'' as described above. 

The root of the proof tree is the initial mask, such as $Q({\rm nil},n).m$, corresponding to our classification problem. 
The classification proof is finished when there are no longer any active nodes. 

Our measure of the size of the proof is to count the number of passive nodes. We decide not to count the
impossible or done nodes, although that would also be a valid choice that could be made, leading to a different notion of minimality
of a proof.

The strategic question for creating the proof is to decide which choice of available cut $(x,y)$ to make starting from 
any given position. The aggregate collection of these choices determines the proof tree and hence the proof. 
We would like to train a deep learning machine to make these decisions.

\subsection{Reinforcement learning for proofs}

We now discuss in a general way the problem of learning to do proofs. We would like to train a machine to predict what is the
best next strategy to use at any stage of a proof. By ``best'' here we mean the strategy that will serve to minimize the
number of steps needed to complete the proof. In our current framework, the proof is always going to finish in a bounded
time, namely after we have done cuts at all the locations. This represents an important distinction from a more general
theorem-proving setup where it might not be clear when or how to finish the proof at all. Having this feature simplifies
the problem for us. 

One main property of the ``theorem-proving'' goal, which is maintained in our situation, is the fact that the number of steps
needed to complete the proof depends on the strategy that we are going to use. We are asking the machine
to predict and minimize a quantity that depends on the configuration of the machine itself. Thus, the question falls into the
domain of {\em reinforcement learning}. 

The main recent advances in this direction are \cite{AlphaGo,AlphaZero}. Of course, the situation of doing a proof is easier
than that of an adversarial game since one doesn't need to consider the possible moves of the adversary. 

There are many online resources available to explain
the general framework that should be used. I found the ``Flappy Bird'' tutorial \cite{FlappyBird} to be particularly helpful.
I would like to mention that that was only one of the literally hundreds of snippets of explanation, code samples, and general
discussions found on the web on a daily basis, which were essential to learning about the materiel and programming
environments used here. Unfortunately,
these were too great in number to be able to record them all as references, and for that  I heartily apologize to and thank the contributors. 

We'll explain in more detail in Section \ref{networks} the utilisation of a pair of neural networks, trained to predict
and minimize the size of proofs. Let us just recall here a few salient aspects of the reinforcement learning process. 

As the machine is basically asked to predict quantities that it has a large hand in determining, one must be careful to avoid
two main potential pitfalls:

\begin{itemize}

\item The machine could fall into a stable situation where very wrong predictions lead to very wrong strategies that 
nonetheless look optimal because of the wrong data that is thereby generated; and

\item As the machine narrows the strategy, we hope, towards a good one, it tends to generate a lot of data on the particular
positions that show up in this strategy, possibly making it look like other positions about which less is known, might be better. 

\end{itemize}

These problems were first explained to me by D. Alfaya, in conversations that occurred well before the current project was envisioned. 

Good exploration and sampling methods need to be chosen in order to mitigate these problems. We note that the second problem
is probably always present to some degree, and it can be seen rather clearly in the results presented graphically at the end
of the paper: when the networks get the node count down to some kind of small value it starts bouncing back a lot. Getting
it to stabilize at the minimal value is much more challenging. 

Mitigating the first problem requires the notion of {\em exploration}: when generating training data, we shouldn't try
only to follow the apparently (by current knowledge of the machine) optimal strategy. Instead, we should generate
training data by following different strategies with various degrees of randomness aiming to explore as much as
possible the full space of possible positions. 

Another question is how far we need to go towards the ends of a proof tree. In principle, with a perfect learning process, it should
be sufficient to just simulate individual proof steps. Indeed, the network is trained to predict the number of remaining nodes,
and this function should have the property that the current number of remaining nodes is the sum of the numbers of nodes below
each of the subsequent positions generated by the best choice of cut, plus $1$ for the upper node itself. We use this method
of generating samples but augment it by running full proofs pruned by dropping certain branches along the way, giving a 
quicker approach to the node values for positions from early in a proof. See Section \ref{samples} for more details. 

\subsection{The neural networks}

Our machine is going to consist of two neural networks $(N,N_2)$, each taking as input a multiplexed proof position $p$. 
The first is designed to give a prediction $N(p)$ of the number of nodes in the proof tree starting from the position $p$. 
The second $N_2(p)$ is an array output consisting of $N_2(p; x,y)$, designed so that $N_2(p; x,y)$ predicts the
sum of the $N(p_i)$ where $p_1,\ldots , p_k$ are the positions generated from $p$ by cutting at $(x,y)$. 

This division of labor is analogous to the value and policy networks of \cite{AlphaGo, AlphaZero}. The utility behind it
is: (1) in doing a large proof $N_2$ provides a fast answer to the question of choosing an optimal cut at each
stage; (2) whereas $N$ provides a fast way of creating training data for $N_2$. 

We program the neural networks using Pytorch. Numerous different options for network architecture were tried. In the
current version of the program, most of the middle layers are convolutional on a $2d$ array 
\cite{DeepLearning,convolutional2, convolutional}. We note that the input tensors
describing a position are multiplexed as described above, and they can have various dimensions. The choice of 
$2d$ was settled on as a space in which to do a reasonable amount of ``thinking'' while not making the number of trainable
parameters explode too much. Grouped convolution is used to further control the number of parameters. 
Fully connected layers are used just before the output, so that the output array takes into account
the full convolution result in its globality rather than just transfering the convolution result-array to the output array of $N_2$. 

See Section 
\ref{architecture} for more details on the network architecture. The architecture we are currently using is something that
can be changed pretty readily. Having already gone through numerous iterations, the process of settling on a good choice is
by no means closed and this is an area for further work and experimentation.

\subsection{Results and questions}

Basically, the process is able to learn to do proofs. This can be seen in the graphs of the training and proof results
given in Section \ref{runs}. Furthermore, for a small first case where $(a,b)=(3,2)$ (and including
some additional filters, see \ref{addfilt}), we are able to calculate in Section \ref{minimality} 
a proven lower bound for the number of nodes in 
the proof. The neural networks are able to find the best proofs, in other words proofs with that minimum number of nodes. 

It should, however, be pointed out that these results are obtained by training the networks on the same proof task that is being measured. 
It is reasonable to ask how well the theorem-proving knowledge generalizes. 

Some experiments were done with training on certain proofs and then testing on others
(see \ref{generalization} for an example), but those results were not all that great.
The proofs on which the networks were not trained, were sometimes done with a small number of nodes, but on the other hand
they sometimes oscillated with rather large node numbers in an apparently unpredictable way, and the oscillation didn't seem to go away
with further training. One might even suspect that after a certain level of training, the machines were learning to memorize the
proof positions on which they were training, to the detriment of success on not-trained-for proofs. 

We could comment that such ``memorization'' might be possible for smaller cases such as $(3,2)$, but for some of the larger
cases that we were able to treat, the number of nodes occurring in a given proof was bigger than the number of trainable parameters
of the model (see \ref{graphOther}) so the results don't seem to be systematically ascribable to simple memorization of positions. 

The question of obtaining machines that are better able to generalize from one proof situation to another, seems like a difficult
question for further research. It doesn't seem clear, for example, what kind of training parameters could be used to favorize that. 

As a variant on the ``generalization'' question, one could ask whether training for certain proofs, then using the resulting
network state as a ``warm start'' for training on different proofs, would produce a noticeable positive effect. There are very preliminary indications
in that direction, but we don't have firm data.

\subsection{Acknowledgements}

This work has been supported by the French government, through the 3IA Côte d’Azur Investments in the Future project managed by the National Research Agency (ANR) with the reference number ANR-19-P3IA-0002.

This project received funding from the European Research Council (ERC) under the European Unions Horizon 2020 research and innovation program 
(Mai Gehrke's DuaLL project, grant agreement 670624).

This research was supported in part by the International Centre for Theoretical Sciences (ICTS) during a visit for participating in the program   
``Moduli of bundles and related structures'' (ICTS/mbrs2020/02). 

This material is based upon work supported by a grant from the Institute for Advanced Study.

\medskip

The program in Pytorch was developed and run with the help of Google Colaboratory. 

\bigskip

I would like to thank the many people whose input into this work has been essential. 
This work fits into a global project with participation by, and/or discussions with: 
David Alfaya, Edouard Balzin, Boris Shminke, Samer Allouch, Najwa Ghannoum, Wesley Fussner, Tom\'a\v{s} Jakl, Mai Gehrke,  
Michael and Daniel Larsen, and
other people. I would like to thank Sorin Dumitrescu for the connection to the 3ia project that started off this research. 
Discussions with Paul-André Melliès, Hugo Herbelin and Philippe de Groote provided important motivation. 
Other valuable comments have come from discussions with Geordie Williamson, Pranav Pandit, Daniel Halpern-Leistner, 
Charles Weibel, Paul Valiant, François-Xavier Dehon, Abdelkrim Aliouche, and 
André Galligo and his working group on AI, in particular talks by Pierre Jammes and 
Mohamed Masmoudi. I would like to thank Jean-Marc Lacroix and Roland Ruelle for their help. I would particularly like to thank
Alexis Galland, Chloé Simpson and Léo Simpson for many discussions about machine learning, optimization, and programming.

\section{Nilpotent semigroups}
\label{nilpotent}

By a {\em ${}^0$semigroup} we mean a semigroup with a distinguished element $0$ having the properties
$0x = x0 = 0$ for all $x$. If it exists, $0$ is unique. 
The case of semigroups (without $0$) may be recovered by the operation of formally adding on a $0$
denoted $A\mapsto A^0:= A\sqcup \{ 0\}$. 

If $A$ is a set we denote by $A^0$ the set $A \sqcup \{ 0\}$ considered as an object of the category of pointed sets.
Given pointed sets $(A,0)$ and $(B,0)$ we denote the {\em product} as 
$$
(A,0)\otimes (B,0):= (A-\{ 0\}) \times (B-\{ 0\}) \sqcup \{ 0\} .
$$
This is the cartesian product in the category of pointed sets. The sum in that category is
$$
(A,0)\vee (B,0):=  (A-\{ 0\}) \sqcup (B-\{ 0\}) \sqcup \{ 0\} .
$$
The category of pointed sets with these operations is sometimes known as the {\em category of $\ff _1$-modules}
and we follow the line of motivation implied by this terminology. In particular, an $\ff _1$-algebra is going to 
be a pointed set $(A,0)$ together with an associative operation
$$
(A,0)\otimes (A,0) \rightarrow (A,0).
$$
This is equivalent to the notion of a ${}^0$semigroup so an alternate name for this structure is ``$\ff _1$-algebra''.

Whereas conceptually we think primarily of pointed sets, in practical terms it is often useful to consider only the
subset of nonzero elements, so these two aspects are blended together in the upcoming discussion. 
In particular, the ``rank'' of an $\ff _1$-module is the number of nonzero elements. 

A finite semigroup $X$ is {\em nilpotent} if it has a $0$ element and there is an $n$ such that $x_1\cdots x_n = 0$ for
any $x_1,\ldots , x_n$. This condition implies that $0\cdot x = x \cdot 0 = 0$, so 
a nilpotent semigroup is also a ${}^0$semigroup and we can say ``nilpotent semigroup''
in place of ``nilpotent ${}^0$semigroup''. 

Suppose $X$ is a nilpotent semigroup. 
We define $X^k$ to be the set of products of length $k$. We have $X^i = \{ 0\}$ for $i \geq n$ (the smallest $n$ above). 
By definition $X^1=X$. We have
$$
X^{i+1}\subset X^i
$$
and these inclusions are strict for $i<n$. That implies that $n\leq |X|$ in the nilpotency condition. 

We introduce the {\em associated-graded semigroup} $Gr(X)$ defined as follows: the underlying set is 
the same, viewed as decomposed into pieces
$$
Gr(X) := X = \{ 0\} \cup \bigcup_{i\geq 1} (X^i - X^{i+1}).
$$
Put $Gr^i(X):= (X^i - X^{i+1})$, and $Gr^i(X)^0:= Gr^i(X)\cup \{ 0\}$. This notion keeps with the $\ff_1$-module philosophy. 

In the current version of this project, we classify semigroups that are already their own associated-gradeds.
This amounts to saying that the product of two elements in $(X^i - X^{i+1})$ and $(X^i - X^{i+1})$
is either in $(X^{i+j} - X^{i+j+1})$ or is equal to zero. For the $4$-nilpotent case, one can recover the
general classification by just lifting all products that are equal to zero, into arbitrary elements of $X^3$. 

\subsection{The $4$-nilpotent case}

We consider the following situation: we have sets $A$ and $B$ of cardinalities denoted $a$ and $b$
respectively, and we look for a multiplication operation
$$
m:A\times A \rightarrow B^0,
$$
recall $B^0:= B\sqcup \{ 0\}$. We require that $B$ be contained in the image
(it isn't necessary to have $0$ contained in the image, that might or might not be the case).

Such an operation generates an equivalence relation on 
$(A\times A \times A)^0$ in the following way. If $m(x,y)=m(x',y')$ then for 
any $z$ we set $(x,y,z)\sim (x',y',z)$ and $(z,x,y)\sim (z,x',y')$. Furthermore, if
$m(x,y)=0$ then we set $(x,y,z)\sim 0$ and $(z,x,y)\sim 0$. 

Let $Q$ be the set of nonzero equivalence classes (it could be empty). We obtain a graded semigroup structure 
on 
$$
X = A^0 \vee B^0 \vee Q^0 = A \sqcup B \sqcup Q \sqcup \{ 0\} .
$$
Suppose given a graded semigroup of the form $A^0\vee B^0 \vee P^0$. Let $m$ be the multiplication operation
from $A\times A$ to $B^0$ and suppose its image contains $B$. 
This yields $Q^0$. There is a unique map $Q^0\rightarrow P^0$ inducing a morphism
of graded semigroups. 

Because of this observation, we would like to classify multiplication maps 
$m:A\times A\rightarrow B^0$ such that the quotient $Q:= A^3 / \sim$ has at least two elements.

Given a multiplication operation satisfying the conditon $|Q| \geq 2$, we can consider a quotient $Q^0 \rightarrow I^0= \{ 0, 1\}$
that sends $Q$ surjectively to $I^0$. 
Roughly speaking, knowing $Q^0$ corresponds to knowing these quotients (there is certainly a form of Stone duality going on here). 
Therefore, we try to classify graded semigroups of the form
$$
A^0 \vee B^0 \vee I^0
$$
which is to say, $4$-nilpotent graded semigroups of size vector $(a,b,1,1)$. 

For $b\leq 2$ and general $a$, we comment that D. Larsen has a sketch of classification leading to a formula for the cases
$(a,1,1,1)$ and $(a,2,1,1)$. 

We are going to assume $b\geq 2$ (the case $b=2$ being nonetheless an interesting one from the proof-learning
perspective), and 
a few further restrictions to the classification setup will be imposed, as discussed in \S \ref{addfilt} below. 

A next question is the choice of ordering of the sets $A$ and $B$. For the computer program, a set
with $a$ elements is $A=\{ 0,\ldots , a-1\}$. Similarly $B=\{ 0,\ldots , b-1\}$. We let the ``zero''
element of $B^0$ correspond to the integer $b$ so $B^0=\{ 0,\ldots , b\}$ containing the subset $B$ indicated above. 
We'll denote this element by $0_B$ in what follows, in other words $0_B$ corresponds to the
integer $b\in \{ 0,\ldots , b\}$.

Our structure is therefore given by the following operations:
$$
\mu :A\times A \rightarrow B^0
$$
$$
\phi : A \times B^0 \rightarrow I^0
$$
and 
$$
\psi : B^0\times A \rightarrow I^0
$$
with the last two satisfying $\phi (x,0_B) = 0$ and $\psi (0_B,x) = 0_I$ for all $x\in A$. They are subject to the condition
that the combined multiplication operation should be associative.

\section{A sieve reduction}

The classification proof setup that we have adopted is to fix the $\phi$ matrix and divide by permutations of 
$A$ and $B$. That is to say, $\srS _a \times \srS_b$ acts on the set of matrices (i.e. $a\times b$ matrices with
boolean entries) and we choose a representative for each orbit by a sieve procedure. This results in a reasonable
number of cases, and the sieve procedure also gives a light property of ordering on the elements of $A$ (resp. $B$)
that seems somewhat relevant. 

Once this matrix is fixed, we search for matrices $\mu$. We get some conditions on the matrix $\psi$
and these are combined into a ternary operation
$$
\tau : A\times A \times A \rightarrow I^0.
$$
The proof is by cuts on the possibilities for the matrix $\mu$, and the leaf of the proof tree is declared to be `done'
when $\mu$ is determined. It doesn't seem to be necessary or particularly useful to make cuts on the possibilities
for the matrices $\psi$ or $\tau$ although this could of course be envisioned.\footnote{We also don't count the potential steps
that might be needed to determine $\psi$ once $\mu$ is fixed, that seems to be mostly negligeable for the sizes
under consideration.}

Another possibility, for absorbing the $\srS _a \times \srS_b$ action, would be to declare that the
values of $\mu$ in $B$ should be lexicographically ordered as a function of $(x,y)\in A\times A$, and then
choose representatives for the initial elements under the $\srS_a$ action. Many attempts in this direction were made,
but in the end, it seemed to be less useful than the current setup. 

We therefore start with a set $\Sigma$ of input data. Each datum in $\Sigma$ consists of a representative for the
equivalence class of a function $\phi : A\times B^0\rightarrow I^0$ (such that $A\times \{ 0_B\}$ maps to $0_I$)
under the action of the group $\srS _a \times \srS_b$. The equivalence class is chosen in a way
that generally puts the $1$'s in this matrix towards lower indices in $A$ and towards the higher indices in $B$
(see for example the instances displayed in the next subsection). 

The other blocks of the input datum, for the right multiplication $\psi : B^0 \times A \rightarrow I$ and
the product $\mu : A\times A \rightarrow B^0$, 
are left free at the root of the proof. 

Here is a table of the sizes $| \Sigma |$, that is to say the numbers of equivalence classes, in terms of $a$ and $b$.
This is the table of \cite{oeis-table}, see the references on that page, and others such as \cite{Harary}.
For $(6,6)$, the value is taken from \cite{oeis-table}. 

\medskip

\noindent
\hspace*{3cm}
\begin{tabular}{|c|c|c|c|c|c|}
\hline
$a\; \backslash \; b$ & $2$ & $3$ & $4$ & $5$ &$ 6$ \\
\hline
$2$ & $7$ &  $13$ & $22$ & $34$ & $50$ \\
\hline
$3$ & $13$ & $36$ & $87$ & $190$ & $386$ \\
\hline
$4$ & $22$ & $87$ & $317$ & $1053$ & $3250$ \\
\hline
$5$ & $34$ &  $190$ & $1053$ & $5624$ & $28576$ \\
\hline
$6$ & $50$ & $386$ & $3250$ & $28576$ & $251610$ \\
\hline
\end{tabular}

\bigskip

We are not very interested in the function that sends everything to $0_I$.
Furthermore we typically don't consider the first few elements of $\Sigma$ that correspond to cases where 
$\phi (x,y) = 0_I$ for almost all values of $x$. These correspond to initial conditions with a large symmetry
group, that are partially absorbed by the symmetry consideration that is explained next. The cases that aren't 
covered by the symmetry consideration should be treated by also specifying the matrix $\psi$; we don't
pursue that at the present time. The remaining values of $\sigma$, constituting most of them, are 
considered as the ``suggested instances'', cf the end of \ref{addfilt}. 

We exploit the symmetry obtained by interchanging the order of multiplication: given a semigroup
$X$ one gets the opposite semigroup $X^o$ with the same set but composition $\ast$ defined 
in terms of the composition $\cdot$ of $X$, by
$$
x\ast y := y\cdot x .
$$
This interchanges the matrices $\phi$ and $\psi$, notably. 
We define an additional filter  ``half-ones'' (cf \ref{addfilt}) to make the following assumption:
\newline
---That the number of nonzero entries of the matrix $\psi$ is $\leq$ the number of nonzero entries
of $\phi$. (The latter number being fixed by the choice of element $\phi \in \Sigma$).

\subsection{Initial data for $(3,2)$}
\label{init32}

For reference we record here the left multiplication matrices for the $13$ initial instances in $\Sigma$ given by the sieve 
for $(a,b)=(3,2)$. Recall that the left multiplication is the product $A\times B^0 \rightarrow I$, but the product with the
zero element (numbered as $2\in B^0$ here) is zero so we only need to include the first two columns. This is a 
$3\times 2$ matrix. The entry $L_{ij}$ is the product $i\cdot j$ for $i\in A$ and $j\in B$ ($i\in \{ 0,1,2\} $ and $j\in \{ 0,1\}$).
$$
L(\sigma = 0) = 
\left[
\begin{array}{ccc}
0 & 0 \\
0 & 0 \\
0 & 0 
\end{array}
\right]
\;\;\;\;
L(\sigma = 1) = 
\left[
\begin{array}{ccc}
0 & 1 \\
0 & 0 \\
0 & 0 
\end{array}
\right]
\;\;\;\;
L(\sigma = 2) = 
\left[
\begin{array}{ccc}
0 & 1 \\
0 & 1 \\
0 & 0 
\end{array}
\right]
$$
$$
L(\sigma = 3) = 
\left[
\begin{array}{ccc}
0 & 1 \\
0 & 1 \\
0 & 1 
\end{array}
\right]
\;\;\;\;
L(\sigma = 4) = 
\left[
\begin{array}{ccc}
1 & 1 \\
0 & 0 \\
0 & 0 
\end{array}
\right]
\;\;\;\;
L(\sigma = 5) = 
\left[
\begin{array}{ccc}
1 & 0 \\
0 & 1 \\
0 & 0 
\end{array}
\right]
$$
$$
L(\sigma = 6) = 
\left[
\begin{array}{ccc}
1 & 1 \\
0 & 1 \\
0 & 0 
\end{array}
\right]
\;\;\;\;
L(\sigma = 7) = 
\left[
\begin{array}{ccc}
1 & 0 \\
0 & 1 \\
0 & 1 
\end{array}
\right]
\;\;\;\;
L(\sigma = 8) = 
\left[
\begin{array}{ccc}
1 & 1 \\
0 & 1 \\
0 & 1 
\end{array}
\right]
$$
$$
L(\sigma = 9) = 
\left[
\begin{array}{ccc}
1 & 1 \\
1 & 1 \\
0 & 0 
\end{array}
\right]
\;\;\;\;
L(\sigma = 10) = 
\left[
\begin{array}{ccc}
1 & 1 \\
1 & 0 \\
0 & 1 
\end{array}
\right]
$$
$$
L(\sigma = 11) = 
\left[
\begin{array}{ccc}
1 & 1 \\
1 & 1 \\
0 & 1 
\end{array}
\right]
\;\;\;\;
L(\sigma = 12) = 
\left[
\begin{array}{ccc}
1 & 1 \\
1 & 1 \\
1 & 1 
\end{array}
\right]
$$

\section{The classification task}
\label{task}

Given three sets $X,Y,Z$, 
a {\em mask} for a function $X\times Y \rightarrow Z$ is a boolean tensor $m:X\times Y \times Z \rightarrow \{ 0, 1\}$. 
A function $f:X\times Y \rightarrow Z$ is covered by $m$, if $m(x,y,f(x,y)) = 1$ for all $x,y$. 

The statistic of $m$ denoted ${\rm stat}(m)$ is the function $X\times Y\rightarrow \nn$ sending $(x,y)$ to the
number of $z$ with $m(x,y,z)=1$, i.e. it is the sum of $m$ along the $Z$ axis. 

If ${\rm stat}(m)(x,y) = 0$ for any pair $(x,y)$ then there doesn't exist a covered function. Thus, we say that $m$
is {\em possible} if ${\rm stat}(m) > 0$ at all $x,y$. 

If ${\rm stat}(m)(x,y) = 1$ it means that there is a single value $z$ such that $m(x,y,z)=1$. This means that if
$f$ is a covered function, we know $f(x,y)=z$. 

We say that the mask is {\em done} if ${\rm stat}(m)(x,y) = 1$ for all $x,y$. In this case it determines a unique function $f$.

Given masks $m,m'$ we say that $m\subset m'$ if $m(x,y,z) = 1 \Rightarrow m'(x,y,z) = 1$. 

We now get back to our classification task. A {\em position} is a quadruple of masks $p=(m,l,r,t)$ where 
$m$ is a mask for the function $\mu : A\times A \rightarrow B^0$, $l$ is a mask for $\phi$, $r$ is a mask for $\psi$,
and $t$ is a mask for $\tau : A\times A \times A \rightarrow I^0$ (a mask for a ternary function being defined analogously). 

We say that a position $p$ is a subset of another position $p'$ if $m\subset m'$, $l\subset l'$, $r\subset r'$ and $t\subset t'$. 

Given a point $\phi\in \Sigma$, the mask $l$ is required to be equal to the done mask determined by this function. 

Let $\Pp$ be the set of positions and let $\Pp _{\Sigma n'}$ be the set of positions corresponding to 
any subset $\Sigma ' \subset \Sigma$. Let $\Pp _{\phi}= \Pp _{\{ \phi \} }$ be the set of positions
corresponding to a point $\phi \in \Sigma$. 

A position $(m,l,r,t)\in \Pp$ is {\em impossible} if any of the masks $m,l,r,t$ are not possible i.e. have a point where
the statistics are $0$. 

A position is {\em realized} if there is a collection of functions covered by the masks that form a graded nilpotent
semigroup. We may also impose other conditions on the realization such as discussed above. 

We'll define a function
$$
{\rm process} : \Pp\rightarrow \Pp 
$$
such that ${\rm process}(p)\subset p$, and such any realization of $p$ is also a realization of ${\rm process}(p)$.
The process function is going to implement some basic steps of deductions from the associativity condition. 
The functions going into ${\rm process}$ will be shown in Section \ref{appendix}.

The function is repeated until it makes no further changes, so that ${\rm process}({\rm process}(p)) = {\rm process}(p)$.

We also define a function
$$
{\rm filter}: \Pp \rightarrow \{ {\rm active}, {\rm done}, {\rm impossible}\}
$$
by saying that ${\rm filter}(p)$ is impossible if:
\newline
---any of the masks in $p$ is not possible (i.e. there is a column containing all $False$'s);
\newline
---in \S \ref{addfilt} we introduce two additional filters that could also be imposed. 

We say that ${\rm filter}(p)$ is done if the mask $m$ is done. Note here that we aren't necessarily requiring $r$ or $t$ to be done
($l$ is automatically done since it corresponds to the function $\phi$). 

We say that ${\rm filter}(p)$ if it is neither impossible or done.

The initial position corresponding to $\phi \in \Sigma$ consists of setting the mask $l$ to be the done mask corresponding to the function 
$\phi$, and setting $m$, $r$ and $t$ to be full masks i.e. ones all of whose values are $1$. 

The goal is to classify the done positions that are subsets of the initial position. 

We do a proof by {\em cuts}. A cut at a position $p$ corresponds to a choice of $(x,y)\in A\times A$, leading to the
position leaves $p_1,\ldots , p_k$ that correspond to choices of $z_1,\ldots , z_k$ such that $m(x,y,z_i)=1$. 
In each position $p_i$ the column $(x,y)$ of $m$ is replaced by a column with a unique $1$ at position $z_i$,
the rest of $m$ remains unchanged. The generated positions are then processed. 

Clearly we only want to do this if ${\rm stat}(m)(x,y) \geq 2$. Such a choice is available any time ${\rm filter}(p)$ is active. 

A partial classification proof consists of making a series of cuts, leading to a proof tree (see below).
Each leaf is filtered as active, done or impossible. The proof is complete when all the leaves are filtered as
done or impossible. The data that is collected is the set of done position at leaves of the tree. These are the 
done positions subsets of the initial position.

\subsection{Additional filters}
\label{addfilt}

In view of the size of certain proof trees that occur, it has shown to be useful to introduce some additional filters.
They should be considered as acceptable in view of the classification problem. 

One filter that we call the ``profile filter'' asks that there shouldn't be two elements that have the same multiplication
profile, i.e. that provide the same answers for all multiplications with other elements. In other words, if there exist
distinct elements $x\neq x'$ such that $x\cdot y = x'\cdot y$ and $y\cdot x = y'\cdot x$ then this filter marks
that case as ``impossible''. The reasoning is that such examples can be obtained from the classification for
smaller values of $n$ by simply doubling an object into two copies of itself as far as the multiplication table is concerned. 

The other filter called ``half-ones'' imposes the condition that we discussed previously, saying 
the number of nonzero entries of the right-hand multiplication 
matrix $\psi$ should be $\leq$ the number of nonzero entries
of the left-hand one $\phi$, which we recall is fixed by the choice of instance in $\Sigma$. 

To impose the half-ones condition, this filter classifies a position as ``impossible'' if:
\newline
---the mask $r$ contains uniquely defined entries with values $\neq 0_I$ in number larger than the number of
nonzero entries of $\phi$, so that a realization would have to contradict our half-ones condition. 

Entries are provided in the `Parameters' section of the program to turn these filters off or on. By default they are turned on,
and our discussion of numbers of proof nodes below will assume they are turned on unless otherwise specified. 
In a similar vein, the program initialization prints a piece of information about suggested values of the instance $\sigma$, 
namely it gives a list of values for which strictly more than half of the columns are identically zero. It is suggested that
the treatment of these instances, using the symmetry to possibly interchange $\phi$ and $\psi$, should be done separately
for the purposes of a general classification proof. The values of $\sigma$ in this list are not viewed as
cases that necessarily need to be done using the present proof machinery. Therefore, in some of the larger examples
in Section \ref{runs}, 
when speaking of ``suggested locations'' we means instances $\sigma$ that aren't in this list,
i.e. those for which $\leq$ half of the columns in $\phi$ are not identically zero. 

\section{Proof tree}

A proof by cuts leads to a proof tree $\srT$. This is defined in the following way. The definition will be inductive: 
we define a notion of partial proof tree, how to extend it, and when it becomes complete. 

The tree has nodes connected by edges, viewed in a downward manner: each node (apart from the root) has
one incoming edge above it and some outgoing edges below. The leaf nodes are those having no outgoing edges. 

Each node is going to have a position attached to it. Each non-leaf node, also called `passive', has
a pair $(x,y)\in A\times A$ called the `cut' associated to it. 

The root node corresponds to the initial position,
which in our setup comes from the chosen matrix $\phi \in \Sigma$.  Call this position $p_{{\rm root}}(\phi )$. 

The leaf nodes are classified into `active', `done' and `impossible' cases. These depend on the position $p$
associated to the node, via the function ${\rm filter}(p)$ that determines the case. 

The proof tree is said to be complete if all the leaf nodes are either done or impossible. 

If a partial proof tree is not complete, then it can be extended to a new tree in the following way. Choose an active leaf
node corresponding to position $p=(m,l,r,t)$ and choose a cut $(x,y)\in A\times A$. This should be chosen so that 
$k:={\rm stat}(m)(x,y)\geq 2$, such a choice exists because if not then the position would be classified as impossible or done. 

Let $z_1,\ldots , z_k$ be the values such that $m(x,y,z_i)=1$. 

The pair $(x,y)$ is going to be the one associated to the corresponding node in the new tree. The new tree is the same as the
previous one except for the addition of $k$ new nodes below our chosen one, with positions $p_1,\ldots , p_k$. 
These positions are determined as follows. Let $l'_i=l$, $r'_i=r$ and $t'_i=t$. Let $m'_i$ be the same matrix as 
$m$ but with the column $m'_i(x,y,-)$ replaced by a column with a single $1$ at location $z_i$. This defines positions $p'_1,\ldots 
, p'_k$. We then set $p_i:= {\rm process}(p'_i)$. 

\begin{proposition}
Suppose $M$ is a semigroup structure covered by a position $p$ at a node of the tree. Then it is covered by
exactly one of the new nodes $p_1,\ldots , p_k$. Therefore, if $M$ is a semigroup structure covered by the root position 
$p_{{\rm root}}(\phi )$, it is covered by exactly one of the positions corresponding to leaf nodes of the tree. 
\end{proposition}
\begin{proof}
With the above notations at a node corresponding to position $p$, if $M=(\mu , \phi , \psi , \tau )$ is a semigroup
structure covered by $p$ then $\mu (x,y)\in \{ z_1,\ldots , z_k\}$. In view of the replacement of the column $m(x,y,-)$
by $k$ columns corresponding to the values $z_i$, it follows that $\mu$ is covered by exactly one of the masks $m'_i$.
As the other ones $l',r',t'$ are the same as before, we get that $M$ is covered by exactly one of the positions
$p'_i$. Then, the property of the ${\rm process}$ function implies that $M$ is covered by exactly one of the positions $p_i$. 
This proves the first statement; the second one follows recursively. 
\end{proof}

{\em Remark:} There is no semigroup structure satisfying our assumptions and covered by an impossible node. Thus, given
a completed proof tree, the admissible semigroup structures with given $\phi$ are covered by the done nodes of the tree. 

{\em Remark:} 
We note that a single done node could cover several structures, since we are only requiring that the mask $m$ be done; 
the mask $r$ could remain undone and correspond to several different functions $\psi$. This phenomenon is rare.

We will be interested in counting the cumulative number of nodes over the full tree when the proof is completed. 
By this, we mean to count the passive nodes, but not the done or impossible ones.\footnote{Actually, 
the training segment of the program does include a small weight for the impossible or done nodes
with the hope of improving stability, but this isn't counted in the official number of nodes for a proof.}
The count does include the root
(assuming it isn't already done or impossible, which could indeed be the case for a small number of instances of our initial
conditions usually pretty far down in the sieve). 

\section{The best possible choice of cuts}

In order to create a proof, one has to choose the cuts ${\rm cut}(p)=(x,y)$ for each position $p$. The {\em minimization criterion}
is that we would like to create a proof with the smallest possible number of total nodes. This count can be weighted
in a distinct way for the done and impossible nodes. Let $| \srT |$ denote this number, possibly with weights assigned to 
the done and impossible nodes. For the purposes of our discussion, these weights are assumed to be zero, although a very small weight
is included in the implemented program with the idea that it could help the learning process. 

If $v$ is a node of the proof tree, let $\srT (v)$ denote the part of the proof tree under $v$, including $v$ as its root node. 

If $p$ is a position, let $\srT^{\rm min}(p)$ denote a minimal proof tree whose root has position $p$. (There could
be more than one possibility.)  We define the minimal criterion at $p$ to be
$$
\srC ^{\rm min} (p):= | \srT^{\rm min}(p) |.
$$

Suppose $p$ is an active position and $(x,y)$ is an allowable cut for $p$. Let $p_1,\ldots , p_k$ be the new positions generated
by this cut. Define
$$
\srC ^{\rm min} (p; (x,y)) := \srC ^{\rm min} (p_1) + \cdots + \srC ^{\rm min} (p_k).
$$

\begin{lemma}
If $\srT = \srT^{\rm min}(p)$ is a minimal proof tree and $v$ is a node of $\srT$ with position $q$ then
$$
\srT(v) = \srT^{\rm min}(q)
$$
is a minimal proof tree for the position $q$. 

If $\srT^{\rm min}(p)$ is a minimal proof tree with root node having position $p$, which we assume is active,
and if $(x,y)$ is the cut at this root node, then
$$
\srC ^{\rm min} (p) = 1 + \srC ^{\rm min} (p; (x,y)) .
$$
\end{lemma}
\begin{proof}
For the first statement, if one of the sub trees were not minimal it could be replaced by a smaller one and this would
decrease the global criterion for $\srT$, contradicting minimality of $\srT$. Thus, the sub-trees are minimal. 

For the second part, say the nodes below $p$ correspond to positions $p_1,\ldots , p_k$. The tree $\srT^{\rm min}(p)$
is obtained by joining together the sub-trees (that are minimal by the first part) 
$\srT^{\rm min}(p_i)$ plus one additional node at the root. 
This gives the stated count.
\end{proof}

Define the {\em minimizing strategy} as being a function ${\rm cut}^{\rm min}(p)=(x,y)$ where $(x,y)$ is a choice that
achieves the minimum value of $\srC ^{\rm min} (p; (x,y))$ over allowable cuts $(x,y)$ at $p$. We say ``a function'' here because
there might be several choices of $(x,y)$ that attain the minimum so the strategy could be non-unique. 

The fact that the minimization criterion is additive in the nodes below a given node, implies the following---rather obvious---property that says
that if you make a best possible choice at each step along the way then you get a best possible global proof.

\begin{corollary}
If $\srT$ is a proof tree obtained by starting with root node position $p$ and following a minimizing strategy at each node, then
$\srT = \srT^{\rm min}(p)$ is a minimal proof tree for $p$. 
\end{corollary}
\begin{proof}
We'll prove the following statement:
suppose given two different proof trees 
$\srT$ and $\srT'$ that both satisfy the property that they follow a minimizing strategy at each node, then
$| \srT | = | \srT '|$ and both trees are minimal. 

We prove this statement by induction. It is true tautologically at a position that is done or impossible.
Suppose $v$ is a position, and suppose $v_1,\ldots , v_k$ and $v'_1,\ldots , v'_{k'}$ are the nodes below $v$ in
$\srT$ and $\srT '$ respectively. Let  
$p_1,\ldots , p_k$ and $p'_1,\ldots , p'_{k'}$ denote the corresponding positions.
We know by the inductive hypothesis that $\srT (v_i)$ and $\srT' (v'_j)$ are 
minimal trees (because they follow the minimizing strategy), hence 
$$
|\srT (v_i) |= |\srT^{\rm min}(p_i)|  = \srC^{\rm min}(p_i)
$$
and the same for $|\srT '(v'_j) |$. We conclude that
$$
|\srT (v_1) | + \cdots + |\srT (v_k) | = \srC ^{\rm min} (p_1) + \cdots + \srC ^{\rm min} (p_k) =
\srC ^{\rm min} (p; (x,y)),
$$
and again similarly for the $v'_j$. But
$$
|\srT (v) |  = 1 + 
|\srT (v_1) | + \cdots + |\srT (v_k) | 
$$
so combining with the lemma we get 
$$
|\srT (v) |  =  1 + \srC ^{\rm min} (p; (x,y)). 
$$
Similarly 
$$
|\srT '(v) |  =  1 + \srC ^{\rm min} (p; (x',y')). 
$$
The minimizing strategy says that these are both the smallest values among all choices of cuts $(x,y)$. In particular
they are equal, which shows that $|\srT (v) | = |\srT '(v) |$. If we now consider a minimal tree $\srT ^{\rm min}(p)$ 
starting from the position $p$ corresponding to $v$, it starts with a cut $(x'',y'')$ and we have 
$$
|\srT ^{\rm min}(p)| =  1 + \srC ^{\rm min} (p; (x'',y''))
$$
by the lemma. But this value has to be the smallest value among the choices of cuts $(x,y)$ otherwise its value could be
reduced which would contradict minimality of $\srT ^{\rm min}(p)$. Thus, it is equal to the values for $(x,y)$ and $(x',y')$ above,
that is to say
$$
\srC ^{\rm min} (p; (x'',y'')) = \srC ^{\rm min} (p; (x,y)) = \srC ^{\rm min} (p; (x',y')).
$$
Therefore 
$$
|\srT ^{\rm min}(p)| = |\srT (v) |  = |\srT '(v) | .
$$
This shows that $\srT (v)$ and $\srT' (v)$ are also minimal trees. This completes the proof of the inductive statement. 

The inductive statement at the root gives the statement of the corollary. 
\end{proof}

\section{Neural networks}
\label{networks}

The model to be used for learning proofs will consist of a pair of networks $(N,N_2)$ that aim to approximate the 
logarithm of the number of nodes in the best classificatino proof below a given position. The first network
has a scalar output and aims to approximate the logarithm of the number of nodes below the input position,
while the second network has as output an array of size $a\times a$ aiming to approximate, at position $(x,y)$,
the logarithm of the number of nodes below the position we get by cutting at $(x,y)$ from the input position. 

This pair of networks corresponds to the pair of value and policy networks in Alpha Go and Alpha Zero \cite{AlphaGo, AlphaZero}. 
We'll use $N_2$ to decide on the proof strategy, namely by choosing the cut $(x,y)$ that has the smallest output 
$N_2(p;x,y)$ among the allowable cuts. The network $N$ is used, in turn, to train $N_2$. These considerations
are motivated by the fact that calculation of the processing that integrates the associativity axiom does take a nontrivial time
so we wouldn't want, for example, to replace $N_2$ by just evaluating $N$ over all the positions generated by the cuts. 

In more symbolic terms, the networks $N$ and $N_2$
aim to provide approximations 
to ideal functions $\widehat{N}$ and $\widehat{N}_2$ defined as follows:
$$
\widehat{N}(p):= \log _{10} \srC ^{\rm min} (p)
$$
$$
\widehat{N}_2(p;(x,y)):= \log _{10} \srC ^{\rm min} (p; (x,y)).
$$
We then follow the strategy 
$$
{\rm cut}^{N_2}(p):= (x,y) \mbox{  attaining a minimum of } N_2(p;(x,y)).
$$
Clearly, if $N_2$ accurately coincides with $\widehat{N}_2$ then ${\rm cut}^{N_2}(p)={\rm cut}^{\rm min}(p)$ and we obtain
a minimizing strategy. 

The neural networks are trained in the following manner. 
First suppose we are given $N_2$. Then, let $\srT ^{N_2} (p)$ be the proof tree obtained by following 
the cut strategy ${\rm cut}^{N_2}(p)$ (in case of a tie, unlikely because the values are
floating-point reals, the computer determines the minimum here by some algorithm that we don't control). 

Then set $\srC ^{N_2} (p):= |\srT ^{N_2} (p)|$ and let 
$$
\widetilde{N}[N_2](p):= \log _{10}\srC ^{N_2} (p).
$$

On the other hand, suppose we are given $N$. Then for any position $p$ and allowable cut $(x,y)$, let
$p_1,\ldots , p_k$ be the generated positions. We define
$$
\widetilde{N}_2[N](p;(x,y)):= \log _{10}(1 + 10^{N(p_1)} + \ldots + 10^{N(p_k)}).
$$

We would like to train $N$ and $N_2$ conjointly to approximate the values of $\widetilde{N}[N_2](p)$
and $\widetilde{N}_2[N](p;(x,y))$ respectively. The training process is iterated to train $N$, then $N_2$, then $N$, then $N_2$ and so forth.

\begin{theorem}
If such a training is completely successful, that is to say if we obtain $N,N_2$ that give perfect approximations to their target values,
then 
$$
N(p) = \widehat{N}(p)\;\; \mbox{ and } \;\; N_2(p;(x,y)) = \widehat{N}_2(p;(x,y)),
$$
and ${\rm cut}^{N_2}(p) = {\rm cut}^{\rm min}(p)$. Then, the proof tree created using the strategy dictated by $N_2$ is
a minimal one. 
\end{theorem}
\begin{proof}
We define the following inductive invariant $D(p)$ for any position $p$: let $D(p)$ be the maximum depth of any 
proof tree starting from $p$. We note that if $(x,y)$ is any cut allowable at $p$ and if $p_k(x,y)$ denote the
new positions after making that cut (and processing) then $D(p_k(x,y)) < D(p)$. Indeed, given a proof tree for $p_k(x,y)$ with
depth $D(p_k(x,y))$ we can plug it into a proof tree for $p$ that has depth $D(p_k(x,y)) + 1$. 

The possible proof trees have bounded depth, indeed after at most $a^2$ cuts all of the values in the multiplication table 
are determined and any resulting position must be done or impossible. Therefore the maximal value $D(p)$ is well-defined and finite. 

We may now proceed to prove the theorem by induction on the invariant $D(p)$. For a given position $p$, 
consider all the cuts $(x,y)$ allowable at $p$. For each cut, we obtain new positions $p_k(x,y)$. 
The tree $\srT ^{N_2} (p_k(x,y))$ has size $\srC ^{N_2} (p_k(x,y)):= |\srT ^{N_2} (p_k(x,y))|$, the $\log _{10}$ of which is
equal to 
$$
\widetilde{N}[N_2](p_k(x,y)):= \log _{10}\srC ^{N_2} (p_k(x,y)) = \log _{10}|\srT ^{N_2} (p_k(x,y))| .
$$
By the inductive hypothesis, which applies since $D(p_k(x,y)) < D(p)$, we have
$$
\widetilde{N}[N_2](p_k(x,y)) = N(p_k(x,y)).
$$
Putting this into the definition of $\widetilde{N}_2[N]$ we get 
$$
\widetilde{N}_2[N](p;(x,y)):= \log _{10}(1 + 10^{N(p_1(x,y))} + \ldots + 10^{N(p_k(x,y))}) = 
\log _{10}\left  (1 + \sum |\srT ^{N_2} (p_j(x,y))| \right) .
$$
In other words, $10^{\widetilde{N}_2[N](p;(x,y))}$ is the size of the tree that is generated below position $p$ if we choose 
the cut at $(x,y)$. 

As the target value for $N_2(p; x,y)$ is $\widetilde{N}_2[N](p;(x,y))$, our hypothesis now says that 
the size of the tree generated by cutting at $(x,y)$ is $10^{N_2(p;x,y)}$. This shows that
$$
N_2(p;(x,y)) = \widehat{N}_2(p;(x,y)) .
$$

Now, the strategy of the proof is to choose the $(x,y)$ that minimizes $N_2(p;x,y)$. By the previous discussion, this is also
the cut that minimizes the size of the proof tree. 
This shows that ${\rm cut}^{N_2}(p) = {\rm cut}^{\rm min}(p)$.

Therefore, the proof tree created starting from $p$ and following our
strategy, is a minimal one. Therefore, 
$$
\widetilde{N}[N_2](p) = \widehat{N}(p),
$$
and in turn by the hypothesis that $N$ predicts its target value we get 
$$
N(p) = \widehat{N}(p).
$$
This completes the inductive proof of the statements of the theorem. It follows that the proof tree obtained by 
choosing cuts according to the values of $N_2$, is a minimal one.  
\end{proof}

\subsection{Modification by adding the rank}
\label{rankmod}

In our current implementation, we modify the above definitions by an additional term in the function $\widehat{N}_2$, hence in the
training for $N_2$. Let $R(p; (x,y))$ be the normalized rank of $(x,y)$ among the available positions, ordered according to the value of 
$\widetilde{N}_2[N](p;(x,y))$. The normalization means that the rank is multiplied by a factor so it extends from $0$ (lowest rank)
to $1$ (highest rank). We then put
$$
\widetilde{N}^R_2[N](p;(x,y))(p; (x,y)) := \widetilde{N}_2[N](p;(x,y))(p;(x,y)) + R(p; (x,y)).
$$
The network $N_2$ is trained to try to approximate this function. 

Clearly, the theorem works in the same way: if $N_2$ gives a correct approximation to the theoretical value then 
the element of rank $0$ corresponds to the minimal value, and this will also be the minimal value of 
$\widetilde{N}^R_2[N]$. 

This modification is based on the idea of including an element of classification in our training for $N_2$. We recall that
the policy network of \cite{AlphaGo, AlphaZero} was supposed to predict the ``best move'', using a softmax output layer
and being trained with a cross-entropy loss function. 

For us, a pure classification
training would be to try to train by cross-entropy to choose the value of $(x,y)$ that is minimal for $\widetilde{N}[N_2](p)$.
This was tried but not very successfully, the problem being that information about lower but non-minimal values, that could
be of use to the model, is lost in this process. Adding the rank to the score includes a classification aspect, while also not
neglecting the non-minimal but lower values, and expands the extent of the values of the function we are trying to approximate in the
lower ranges. The smaller score values can group very near to the minimal value, meaning on the one hand that some error
in approximating the values can lead to the wrong choice, however it also means that choosing a next-to-best value doesn't
lose too much in terms of size of proof. 

It turns out to be more difficult to approximate the function with addition of the rank, as reflected in the plots of
network output along the training process and the higher loss values for the local network. 
But due to the combination of the two terms, we also need less accuracy in order
to work towards a minimal proof. 

\section{Samples}
\label{samples}

Once we are given $N$, obtaining the sample data for training
$N_2$ to approximate $\widehat{N}_2(p;(x,y))$ is relatively straightforward. Namely, we suppose given some sample
positions $p$, then we choose samples $(p;(x,y))$ (it might not be necessary to include all values of $(x,y)$ for a given $p$)
and the calculation of  $\widetilde{N}_2[N](p;(x,y))$ then just requires applying $N$ to the associated positions
$p_1,\ldots , p_k$. We note here that these are obtained from the raw positions $p'_1,\ldots , p'_k$ (column replacements)
by an application of $p_i := {\rm process}(p'_i)$ so some computation is still involved but in a limited way. 

On the other hand, given $N_2$ to obtain sample data for 
$\widetilde{N}[N_2](p)$ according to our definition, requires a lot of computation.
This is because in the definition of $\widetilde{N}[N_2](p)$ we need to calculate $\srC ^{N_2} (p):= |\srT ^{N_2} (p)|$
which means calculating the whole proof tree $\srT ^{N_2} (p)$. It isn't feasible to do this. We propose here two methods
to get around this difficulty. They both involve further approximation and a recurrent or reinforcement-learning aspect.

\subsection{First method}

The first method is to calculate a pruned proof tree  $\overline{\srT} ^{N_2} (p)$. Here, we only make a choice of cuts
(according to the $N_2$-determined strategy $p\mapsto {\rm cut}^{N_2}(p)$)
to extend the proof tree on some of the active nodes, and prune some other nodes at each stage. In the program it is
called ``dropout'', in other words we drop some of the active nodes. Typically, we fix a number $D$ and at each
iteration, treat at most $D$ active nodes and `drop' the remaining ones. The main way of choosing these is
to choose randomly, however we also envision an adaptive choice of the nodes to keep in order to address the
issue of imbalanced properties of resulting positions, this will be discussed later. 

The proof is then completed significantly faster in the bigger cases (i.e. when $a$ and $b$ 
are strictly bigger than $3$). 
Along the way, this process could also permit to try to gain an estimate\footnote{Attempts to get such an estimate haven't
currently worked very well at all, it seems to be a nontrivial question in highly unbalanced statistics.
In case of success, 
the stochastic dropout proof trees could then be used to estimate the size of the $N_2$-minimizing proof and trigger
an ``early stopping'' of the training process. }
of the size of the fully
completed proof from doing only a very partial one. 

Now we have a proof $\overline{\srT}=\overline{\srT} ^{N_2} (p)$ in which each node $v$ has a position $p(v)$ 
and a subproof $\overline{\srT}(v)$. We would like to define $\widetilde{N}[N_2](p(v))$ by estimating the
value for the full proof 
$$
\srC ^{N_2}(p(v)) = |\srT ^{N_2}(p(v)) |.
$$
The proof $\overline{\srT}(v)$ has a certain number of nodes that remain active since they were dropped; the idea of
the estimation is to use the network $N$ itself to provide an estimate. Namely, if $v'$ is a dropped (hence still active)
node in $\overline{\srT}(v)$ then we add 
$$
10^{N(p(v'))}
$$
into the estimation of the size of $\srT ^{N_2}(p(v))$. This is to say that 
$$
|\srT ^{N_2}(p(v)) |^{\rm estimated}  := | \overline{\srT}(v)| + \sum _{v'} 10^{N(p(v'))}
$$
where the sum is over the dropped nodes $v'$ that are leaves of $\overline{\srT}(v)$. 
We now add the pair 
$$
\left( p(v), \log_{10}\left(  |\srT ^{N_2}(p(v)) |^{\rm estimated}   \right) \right) 
$$
as a sample point in the training data for $N$.

\medskip

\subsection{Second method}

The second method is a one-step version of the first method. Given a position $p$, we use $N_2$ to choose 
the $N_2$-minimizing cut $(x,y)$ for $p$. Let $p_1,\ldots , p_k$ be the positions that are generated from the cut
(including doing processing i.e. $p_i = {\rm process}(p'_i)$ from the raw positions $p'_1,\ldots ,p'_k$). 
Then we put
$$
\widetilde{N}[N_2](p)^{\rm estimated} := \log _{10} \left(
1 + 10^{N(p_1)} + \cdots  + 10^{N(p_k)} \right)
$$
where more precisely if $p_i$ is either done or impossible then the term $10^{N(p_i)} $ is replaced by the
corresponding weight value that we are assigning to this case. We then add the pair
$$
\left( p(v), \widetilde{N}[N_2](p)^{\rm estimated} \right) 
$$
as a sample point in the training data for $N$.

\subsection{Sampling issues}

The first method of generating samples has the property that the generated samples are in the set of positions that
are encountered in an $N_2$-minimizing proof. Notice that the pruned proof tree $\overline{\srT}$ would be a part
of a full proof tree made using the $N_2$-minimizing choice of cuts ${\rm cut}^{N_2}(p)$ at each position $p$. 
This is both useful and problematic. Useful, because in calibrating a minimal proof we are most interested in the positions
that occur in that proof. But problematic because it means that we don't see other positions $p$ that might arise
in a better strategy. Therefore, some exploration is needed. 

The possibility of doing exploration is afforded by the second sampling method. Indeed, it can be done with the same
computation cost starting at any position $p$. Therefore, we can add samples starting from a wide range of positions.

In practice what we do is to add to our pool of positions a random sampling of all the generated positions $p_i$ that
come from various choices of cuts $(x,y)$ that might not be the best possible ones. These may be obtained for example when 
we are doing the computations of samples for training $N_2$. 

The drawback of the second method is that it is entirely reinforcement-based, in other words it doesn't see to the end
of the proof (the importance of doing that was first mentioned to me by D. Alfaya). Theoretically, in the long term after
a lot of training, a pair of networks trained only using the second sampling method should generate the correct values,
however it seems useful to include samples taken according to the first method too as a way of accelerating training. 

The reader will be asking, why not combine the two methods and run a pruned proof for some steps starting from a
position $p$ and collect the sampling data from there. This is certainly another possibility, I don't know whether it can
contribute an improvement to the training. 

We next comment on the adaptive dropout mentioned above. In this problem, some positions generate a very significantly
longer proof than others. Those are the ones that lead to large outcomes in the global proof, so it is better if the
neural networks concentrate their training on these cases to some extent. Therefore, in making a pruned proof tree
$\overline{\srT}$ it will be useful to prune in a way that keeps the nodes that are expected to generate larger sub-trees.
This is measured using the existing network $N$. Thus, in pruning we prioritize (to a certain extent) keeping nodes $v$ that
have higher values of $N(p(v))$. These adaptive dropout proof trees are more difficult to use for estimating the global
size of the proof, or at least I didn't come up with a good method to do that. Thus, the regular stochastic dropout method
is also used, and sampling data is generating using the first sample method from both kinds of proofs. 

Let us mention another technique to add exploration in the first method: we can 
use a randomized choice of proof strategy for the first part of a proof, then switch to the standard $N_2$-based
strategy in the middle, and only sample from nodes that are at or below the switching point.

\section{Network architecture}
\label{architecture}

We describe here the architectures that are used for the neural networks $N$ and $N_2$. 
See \cite{DeepLearning} for the fundamentals.
We didn't do any systematic
hyperparameter search over the possible architectures or possible parameters for the architectures. It could be said,
however, that in the course of numerous iterations of this project, various architectures were tried and the one we
present here seems to be a reasonably good choice with respect to some previous attempts. It is of course likely that
a good improvement could be made.  

The input in each case is the position $p=(m,l,r,t)$ consisting of four boolean tensors of sizes $a\times a \times (b+1)$,
$a\times (b+1)\times 2$, $(b+1)\times a \times 2$ and $a\times a \times a \times 2$ respectively. Recall that we
are interested in the set $B^0$ that has $(b+1)$ elements. Also, $[I^0| = 2$ explaining the values of $2$ in the last
three tensors. In the program, the quadruple of tensors is packaged into a dictionary. 

We remark that $l$ is included even though it isn't the subject of any computation (as it corresponds to the
fixed input $\phi$ at the root of the tree) because the neural network will be
asked to treat several different initial values $\phi$ at the same time. 

The tensors are taken of type \verb+torch.bool+ but are converted to \verb+torch.float+ at the start of  
$N$ and $N_2$. 

We recall that the neural network treats a minibatch all at once, so the tensors have an additional dimension at the start
of size $L:=$\verb+batchsize+.

The output of $N$ is a scalar value, so with the batch it is a tensor of size $L$. The output of $N_2$
is a tensor of size $a\times a$ whose value at $(x,y)\in A\times A$ represents the predicted value if the cut $(x,y)$ is
chosen. Here again an additional dimension of size $L$ is appended at the start, so the
output of $N_2$ is, in all, a tensor of size $L\times a \times a$.

\subsection{Input tensor}

We would like to input the tensors $m,l,r,t$. These have sizes as follows, where $L$ denotes the batchsize: 
$$
\begin{array}{ccc}
\mbox{size } \, m & = & L \times a \times a \times (b+1) \\
\mbox{size } \, l & = & L \times a \times (b+1) \times 2 \\
\mbox{size } \, r & = & L\times (b+1) \times a \times 2 \\
\mbox{size } \, t & = & L \times a \times a \times a \times 2 
\end{array} .
$$
We would like to combine them together into an input vector that will make sense for learning. 

Before doing so, they are converted from boolean values to float values, then normalized so that the sum of the values along
the last dimension is $1$. This normalization, corresponding to viewing the values more like probabilities than like booleans,
was suggested by B. Shminke, see \cite{BalzinShminke}. It seems to improve performance. 

As input into a
fully connected layer, these tensors could just be flattened into vectors (i.e. tensors of size $B\times v$) and concatenated. 
This would give an input of length
$$
v=a^2(b+1) + 4 a (b+1) + 2 a^3.
$$

In order to preserve more of the tensor structure, we prepare the tensors for 
input into a $2$-dimensional convolution layer. For this, we transform them into
tensors of size $L\times f \times a \times a$ by placing the dimensions that are different from $a$ into the first
``feature'' variable, including the middle dimension of $t$ here also, and expanding the $l$ and $r$ tensors by a factor of $a$ to
give them a size that is a multiple of $a\times a$. 
We then concatenate these along the feature dimension. 

More precisely, $m$ will have $(b+1)$ features, $l$ and $r$ have $2(b+1)$ features, and $t$ has $2a$ features. Thus our input
tensor is of size $L\times f \times a \times a$  where $L$ is the batchsize and 
$$
f = 5(b+1) + 2 a.
$$

The choice of method to prepare the input tensor is the first necessary design choice. Many possibilities were envisioned, including
concatenation of various permutations of these tensors, and convolution layers of dimension $1$, $2$ or $3$.

The method we choose to use depends of course on the
processing to follow. The preparation method discussed above is not necessarily the best one but it seems
to work pretty well. 

\subsection{Data flow}

The basic trade-offs that need to be considered are the question of how data flows through the layers, versus 
computational time for a forward pass, and also training time needed for calculation of the gradients and back-propagation. 

We may illustrate this by looking at the idea of starting with a fully-connected or {\em linear} layer. Let's consider for example the case
$(a,b)=(5,3)$. Then the fully flattened input dimension from above is $v=a^2(b+1) + 4 a (b+1) + 2 a^3 = 430$. 
A fully connected layer towards $n$ neurons that would be declared as follows:
$$
\verb+self.linA = nn.Linear(430,n)+
$$
is going to involve $430n$ weight parameters (plus $n$ bias parameters that we don't worry about). If we take, say $n=100$ this
gives $43000$ parameters for this layer only. Furthermore, the problem of size of data goes down from $430$ to $100$ so
it persists. On the other hand, we could take for example $n=32$ and this would give $13760$ weight parameters and yield
a very manageable size of tensor to deal with further. The problem here is that it is trying to pack all the information from
$430$ values (that are either $0$ or $1)$ into a $32$ dimensional space. All further operations will depend only on the
$32$ parameters. 

This constricts the flow of data into the network, something that we would like to avoid. A similar consideration applies at the
output stage. In the middle, some constriction might even be desireable as long as the data can also take a different path by
``skip connections''. 

\subsection{Convolution layers}

It seems that we should be able to do better. If we start with a convolutional layer \cite{DeepLearning,convolutional2,convolutional}, 
we notice (in the same example $(a,b)=(5,3)$)
that the number of features is $f = 5(b+1) + 2 a=30$ spread over a $5\times 5$ array. The first convolutional layers can
therefore with no problem maintain or even increase this number of features. 

The price here is that our set of $a=5$ vectors has an ordering that is not based on any strong natural consideration. Therefore, the
convolutional aspect needs to be higher in order to obtain a good mixing between the features at different locations $(x,y)$. 

It turns out that a weak ordering property is in fact obtained by
our sieve process of choice of the initial matrix $\phi$: the sieve process that we use will make a choice of ordering
for each equivalence class of input, and our implementation has the property that it tends to put more $1$'s as we go towards
one of the corners of $\phi$. This is admittedly not a very convincing justification for why it would be natural to 
use a convolutional structure, so there is undoubtedly room for having a better architecture that takes into account the
tensor property but without requiring an ordering (``graph neural networks'' GNN could come to mind, see 
for example \cite{Galland} and the references therein). 

In order to mix the values between different locations, we choose to use alternating convolution windows of sizes 
$[5,1]$ and $[1,5]$. In a window of size $[5,1]$ it means that the output feature values at location $(x,y)$ are obtained by 
combining together (using the convolution weight matrix)
those of locations\footnote{With circular padding the locations in the window are taken modulo $a$.}
$$
(x-2,y),\;\;
(x-1,y),\;\;
(x,y),\;\;
(x+1,y),\;\;
(x+2,y).
$$
Window size $[1,5]$ means the same but with $y$. In our use
cases, the number of values is $a$, usually $\leq 5$, so convolution with these window values allows data to interact across the
array.

An additional trick is available to reduce the number of parameters: {\em grouped convolution}.
That places the features into groups and applies
a different convolution to each group. A grouped $2d$ convolution with a window of size $[5,1]$, with $8n$ input
features, $8n$ output features and $n$ groups, has
$$
n\cdot 8\cdot 8 \cdot 5 = 320n
$$
parameters, whereas the same non-grouped version would have $320n^2$ parameters. The use of grouped convolution
can allow us to maintain a good size of data flow through a layer, while diminishing the number of parameters and thereby improving the
computational and training time. 

The grouped convolution layer of the previous paragraph is declared as follows:

\noindent
\hspace*{-.5cm}
{\small
\verb+self.convA = nn.Conv2d(8*n,8*n,[5,1],padding = [2,0],padding_mode = 'circular',groups = n)+
}

We note that the group idea can also be applied in a linear layer, by writing the linear layer as a $1d$ convolution that
is going to be applied on a trivial one-element array.

The reader is referred to the program source for the precise specifications of the neural networks. Pytorch notation is
sufficiently self-explanatory that it should be straightforward to understand.

\section{Examples}
\label{runs}

In this section we show the graphs of results of some runs of our networks on various cases. The value of $\sigma$ indicates the
choice of initial instance $\phi \in \Sigma$. The captions include information on the number of trainable parameters in
the global and local networks, and the numbers of iterations in the loops of training operations and sample proofs.

Each graph is done from a ``cold start'', the networks beginning with their standard random initialization. 

\subsection{Training segments}

A given phase of training involves a succession of loops with various mini-batch sizes and numbers of gradient descent steps
per minibatch. Such a ``training segment'' involves the the following steps:

\bigskip

\noindent
\hspace*{.3cm}
\begin{tabular}{|p{1.2cm}|p{1.8cm}|p{2.5cm}|}
\hline
\multicolumn{3}{|c|}{global network $N$} \\
\hline
mini-batches & minibatch size& descent steps per batch \\
\hline
\hline
2 & 20 & 20 \\
\hline
1 & 30 & 30 \\
\hline
2 & 40 & 10 \\
\hline
5 & 60 & 10 \\
\hline
3 & 20 & 8 \\
\hline
3 & 40 & 5 \\
\hline
5 & 30 & 3 \\
\hline
\end{tabular}
\hspace*{2cm}
\begin{tabular}{|p{1.2cm}|p{1.8cm}|p{2.5cm}|}
\hline
\multicolumn{3}{|c|}{local network $N_2$} \\
\hline
mini-batches & minibatch size& descent steps per batch \\
\hline
\hline
3 & 20 & 20 \\
\hline
1 & 30 & 15 \\
\hline
3 & 60 & 4 \\
\hline
3 & 40 & 10 \\
\hline
3 & 20 & 3 \\
\hline
3 & 40 & 2 \\
\hline
3 & 30 & 1 \\
\hline
\end{tabular}

\bigskip

\noindent
In all this gives, for each training segment, the following numbers:
\begin{itemize}

\item
for the global network, $21$ minibatches with $780$ samples and $194$ gradient descent steps

\item
for the
local network, $19$ minibatches with $660$ samples and $135$ gradient descent steps. 

\end{itemize}

All of these choices are arbitrary. The basic idea is to 
overtrain at the start of the training segment, doing a high number of gradient descent steps on a small minibatch, then to 
train more gradually towards the end of the segment. The gradual part is included at the end because the networks will be
used for further proofs (either sampling or doing the actual proofs) only at the end of the training segment. I couldn't say how good
these choices are. 

In the early stages of training, noise was added between the input layer and the network, and the network
weights were perturbed slightly at the outset of each stage. 

In the diagrams shown below, the caption gives the number of basic loop iterations per proof, and the number of training segments per
basic loop. Therefore, the total number of gradient descent steps (resp. samples) in between each pair of proofs
will be the product of the above values ($194$ or $135$, resp. $780$ or $660$) by
the number of basic loop iterations and then by the number of training segments per basic loop.

\subsection{Heuristic benchmark}

For comparison, we give the results of an heuristic benchmark strategy (horizontal line on the graphs). This strategy is obtained by
choosing an ordering of the locations $(x,y)$ and just following the rule of making a cut at the first available location. 
It turns out that a pretty good result is obtained by taking the following order: first $(0,0)$, then $(1,1)$, then
$(1,0)$, then $(0,1)$, then continuing with the remaining $(x,y)$ in lexicographic order starting with $(0,2)$. 
That strategy was found almost by accident. It seems to give a reasonable baseline number for comparison, which in some cases is
the minimum. The learning mechanism is able to do better (or just as well when it is the minimum). 

\subsection{Size $(3,2)$}
\label{graph32}

We start by looking at size $(a,b)=(3,2)$. Consider an easy initial instance, $\sigma = 4$. The number of nodes stabilizes at $11$,
that will be shown to be the minimum in Theorem \ref{theormin}. 

\bigskip

\includegraphics[scale=0.4]{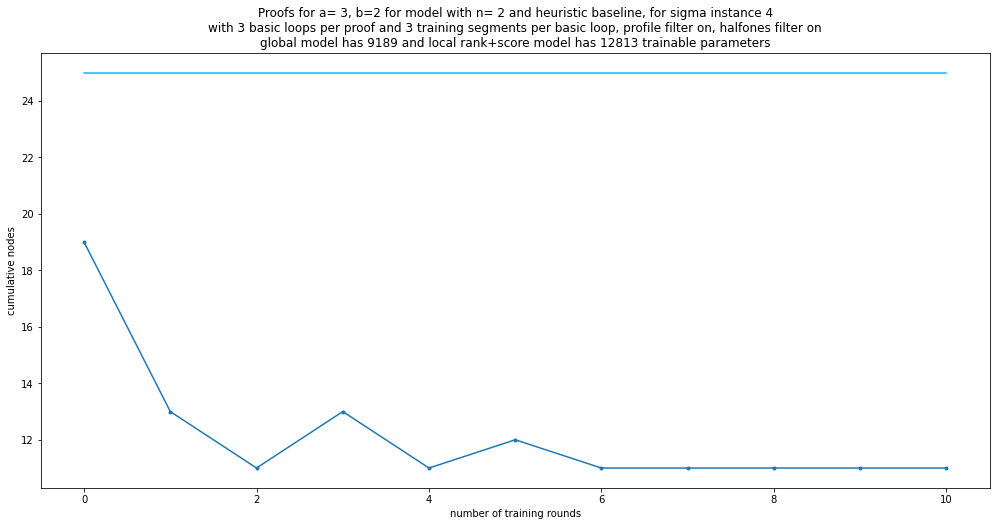}

\includegraphics[scale=0.4]{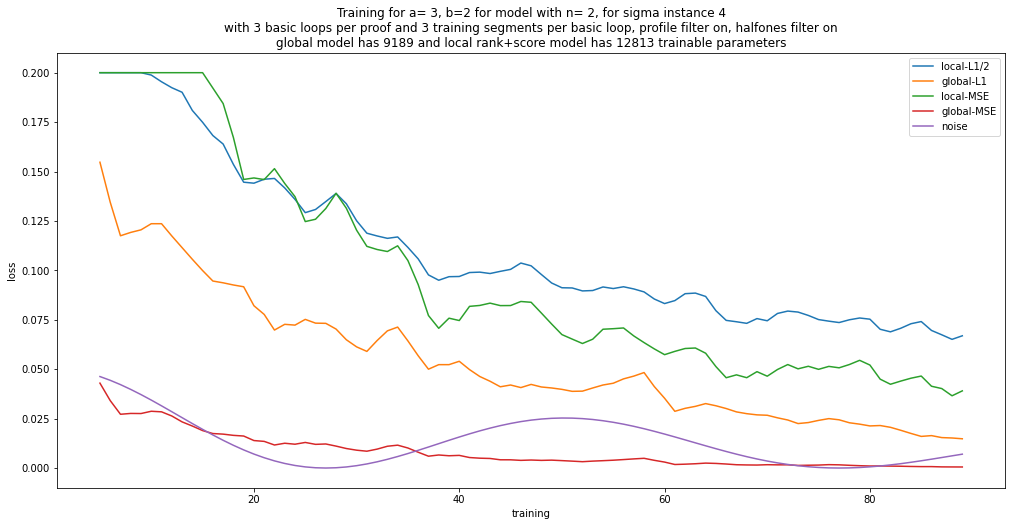}

\bigskip

The first graph pictured above is the number of proof nodes in the sequence of proofs, whereas the second graph gives various loss functions of the
training as the machine evolves. Several loss data points 
are taken in between each pair of
proofs, so the horizontal labels aren't the same. The purple curve represents some noise that is added in the training process,
as we imagine could be useful loosely following \cite{BalzinShminke}. 

\bigskip

In Section \ref{minimality} we'll discuss in more detail all $13$ instances $\sigma$ for size $(a,b)=(3,2)$.
The most difficult case is $\sigma = 3$:

\bigskip

\includegraphics[scale=0.4]{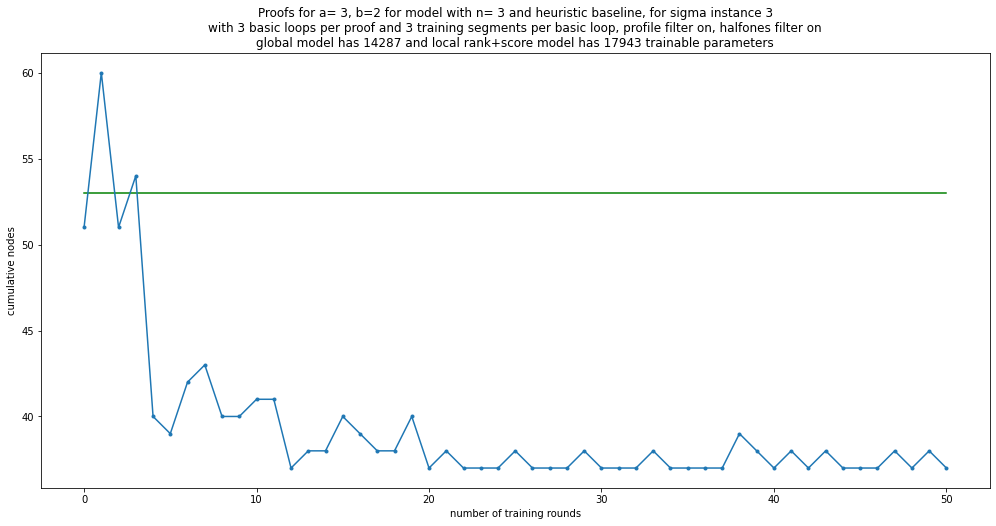}

\includegraphics[scale=0.4]{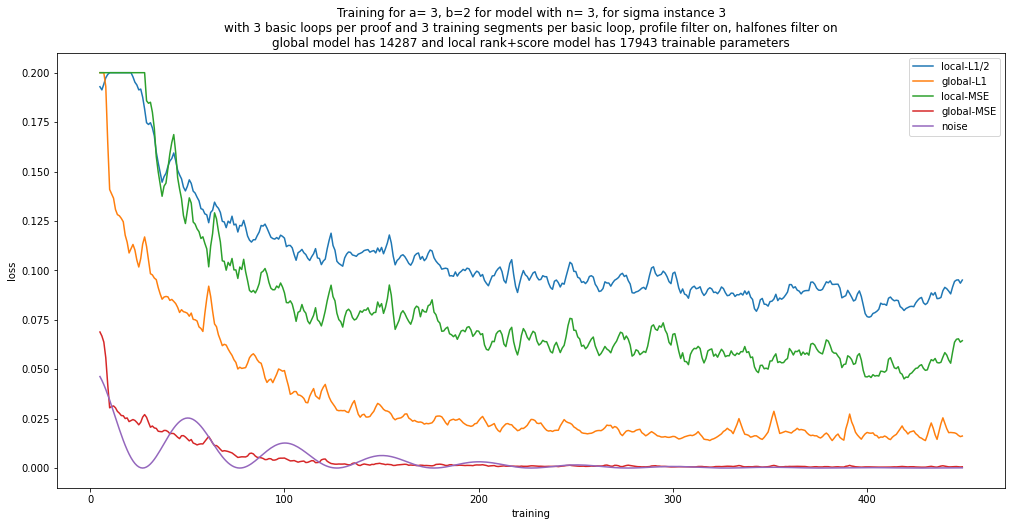}

\bigskip

It looks like the theoretical minimum is $37$, which is indeed the case, as we show in Theorem \ref{theormin} below. 

The other values of $\sigma$ are generally easier to treat individually, yielding pictures more akin to the ones for
$\sigma = 4$. We record next the result of looking at the 
full collection of proofs with all the $13$ initial values of $\sigma$ simultaneously. 

\bigskip

\includegraphics[scale=0.4]{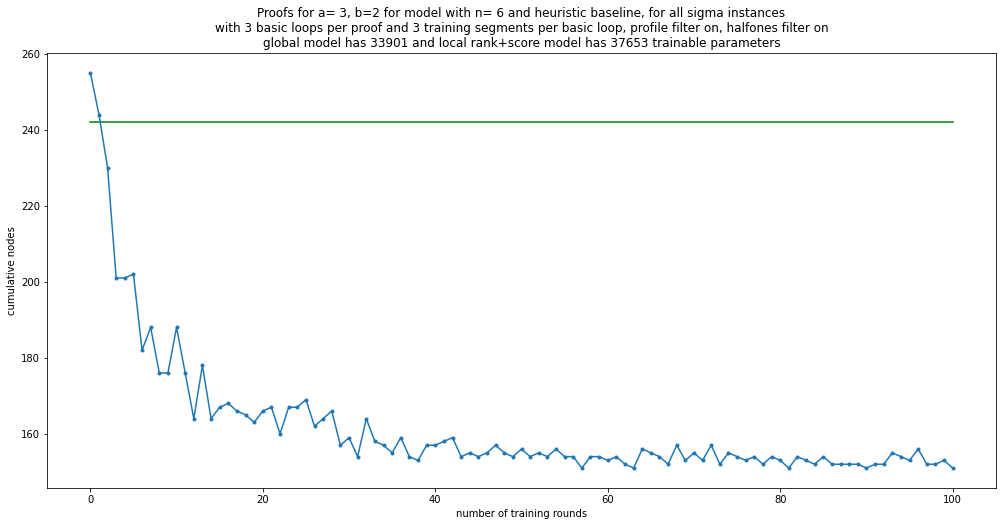}

\includegraphics[scale=0.4]{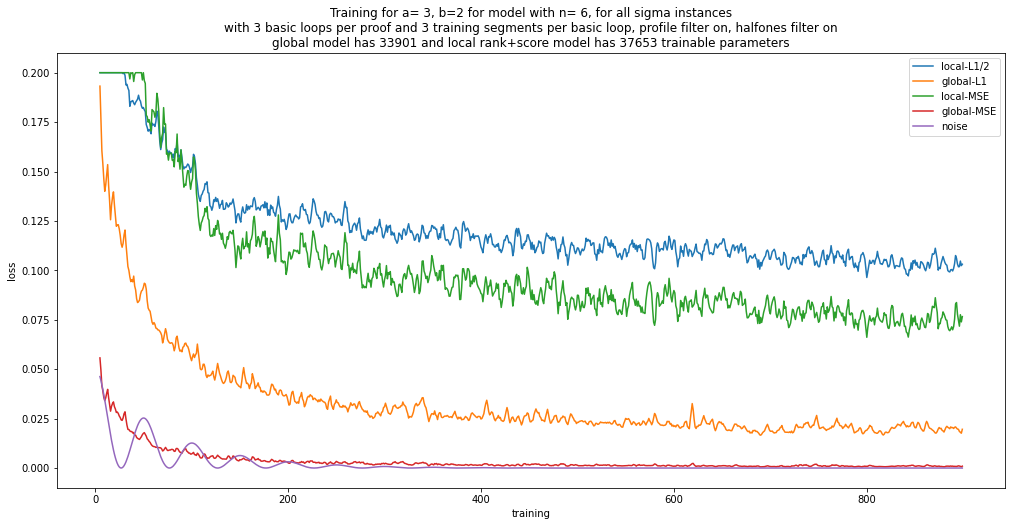}

\bigskip

The theoretical minimum value of $151$ (see Theorem \ref{theormin}) is attained at several proofs (numbers $57$, 
$63$, $81$, $90$, $100$), although the 
model has a tendency to stabilize around a slightly higher value. 

\subsection{Size $(4,2)$}
\label{graph42}

Here is a case for size $(a,b)=(4,2)$, with $\sigma = 5$. 

\bigskip

\includegraphics[scale=0.4]{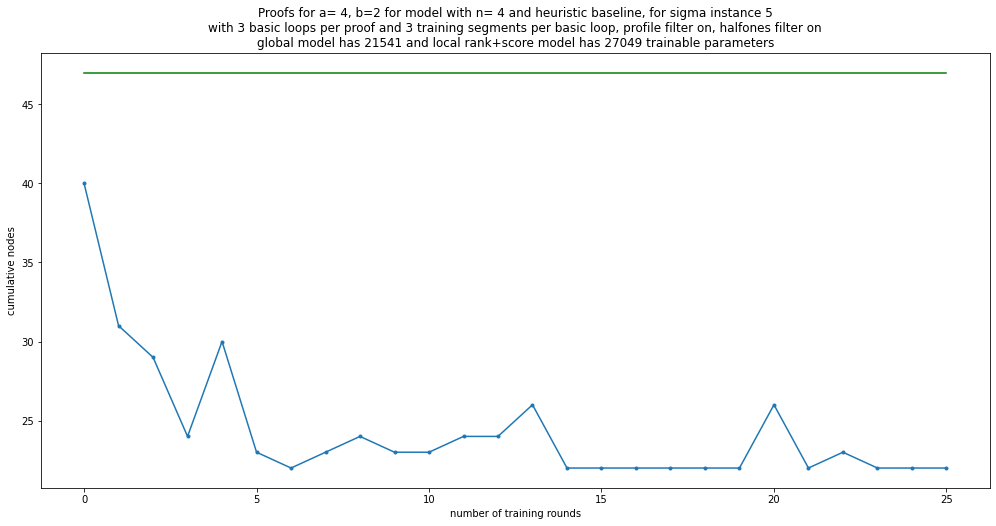}

\includegraphics[scale=0.4]{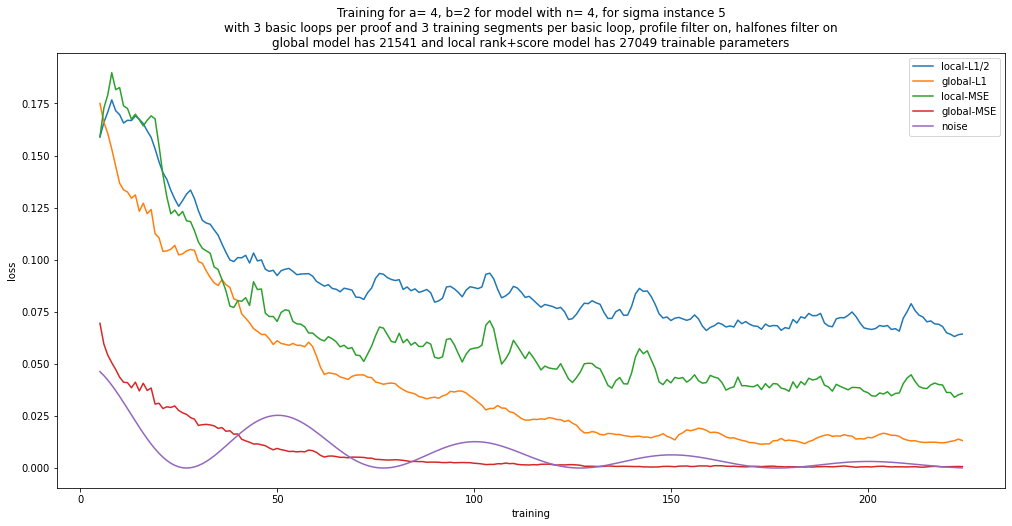}

\bigskip

We conjecture that the theoretical minimum in this case is the value $22$ that is attained first at proof number $6$
and often at the end. 

\subsection{Size $(5,3)$}
\label{graph53}

Here is a sample case for size $(a,b)=(5,3)$, namely $\sigma = 7$, the first one in the range of ``suggested locations''
(cf \ref{addfilt}). 

\bigskip

\includegraphics[scale=0.4]{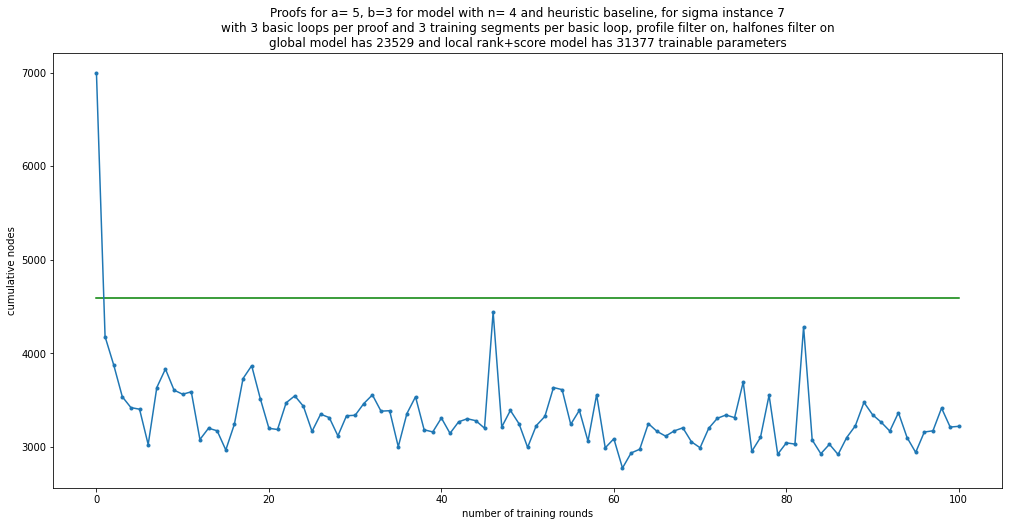}

\includegraphics[scale=0.4]{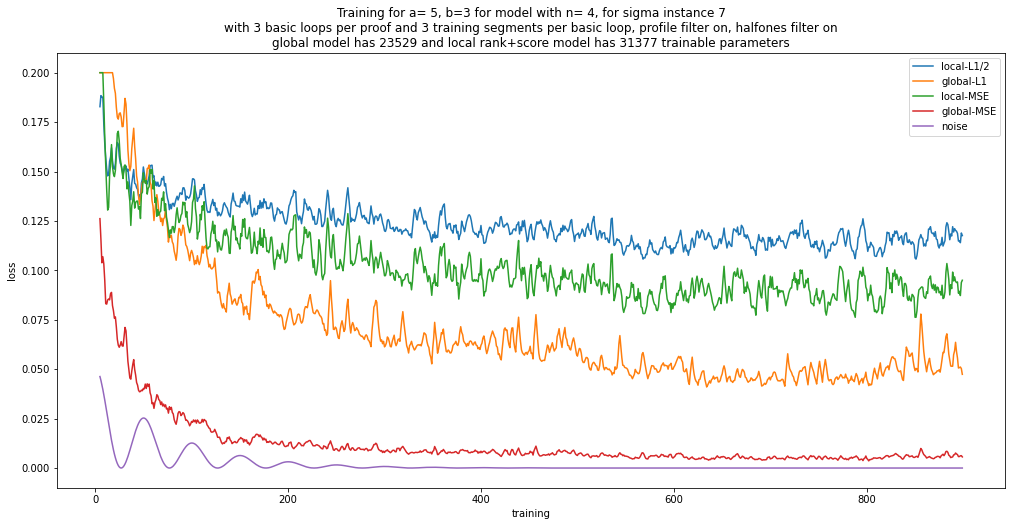}

\bigskip

Here the minimal value attained, at proof number $61$, is $2773$. We don't know how far that might be from the
theoretical minimum.

\subsection{Size $(4,5)$}
\label{graph45}

Here is a sample case for size $(a,b)=(4,5)$, namely $\sigma = 22$. Again, this $22$ is the first in the range of suggested locations
(cf \ref{addfilt}). 
This training and proving process, for a cycle of $50$ proofs, took 2 hours 35 minutes on a Google Colab GPU. 

\bigskip

\includegraphics[scale=0.4]{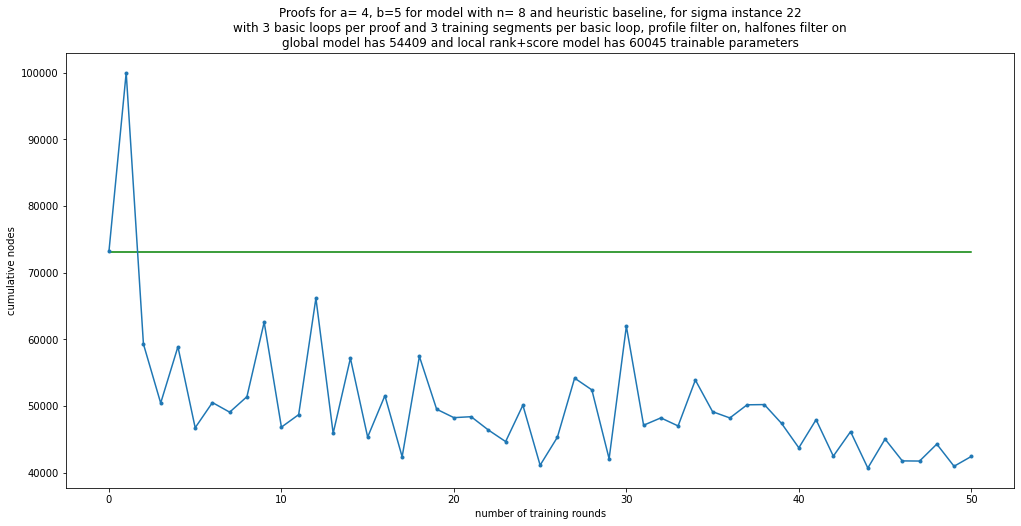}

\includegraphics[scale=0.4]{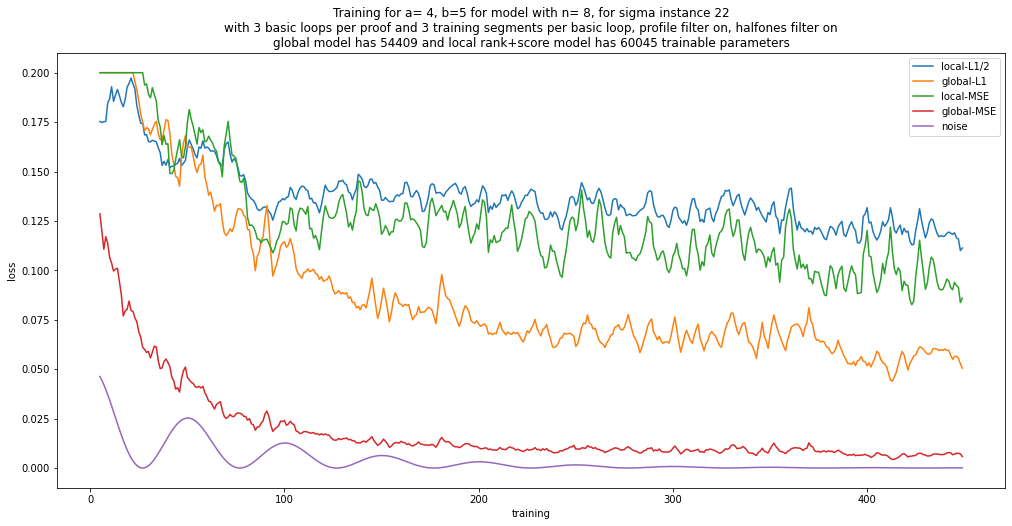}

\bigskip

The minimal value attained, at proof number $44$, is $40720$ (with anew $40984$ at proof $49$). 
We don't know how far that might be from the
theoretical minimum. 

The number of nodes in the proof tree is of the same order of magnitude as
the number of trainable parameters in the neural network. A larger value $n=8$ was chosen for the size of 
the networks in view of the larger size of the problem, leading to $54409$ global and $60045$ local parameters. 
The output data for choice of cut at a node involves at least $2$ and up to $a^2=16$ values, so in all we can say that the number
of parameters for network $N_2$ is significantly smaller than the number of data values required to do a single proof. 

Here are plots of the predictions of the global  and local networks $N$ and $N_2$, after having trained for the cycle $50$ proofs.

\hspace*{3cm}\underline{global} \hspace*{7.5cm}\underline{local}

\hspace*{1cm}
\includegraphics[scale=0.4]{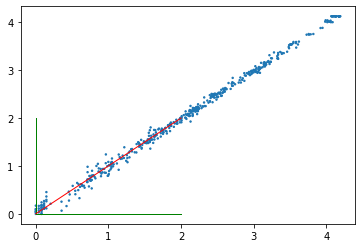}
\hspace*{3cm}
\includegraphics[scale=0.4]{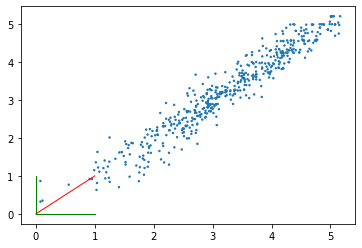}

\noindent
Recall that training data for $N_2$ includes the rank plus the score, see the modification 
\ref{rankmod}, that is why it doesn't appear all that accurate even at this stage. 

\bigskip

The $(4,5)$ case concerned semigroups of size $11$. There are $1053$ possible instances $\sigma$, but the later values are expected
to correspond to shorter proofs, so the instance $\sigma = 22$ that was treated here should be at the high end of proof size.

\subsection{Generalization: skipping a value for training}
\label{generalization}

In order to test the capacity for generalization, we can also do the following: train on all the values in a certain segment,
except skipping one value, then test this by doing proofs at the value that was skipped. Here is the result 
for the case of $(a,b)=(3,2)$, training on all values except $\sigma = 5$ and proving for $\sigma = 5$: 

\bigskip

\includegraphics[scale=0.4]{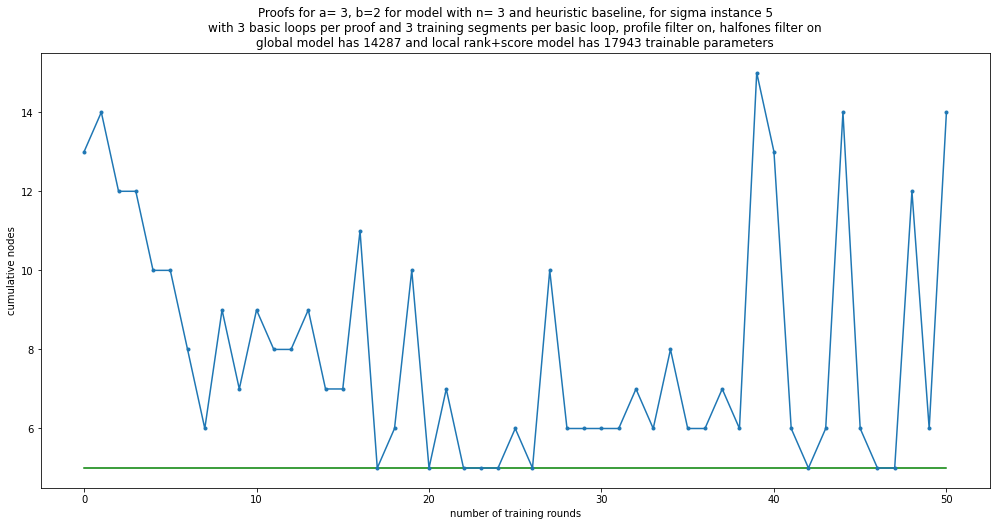}

\includegraphics[scale=0.4]{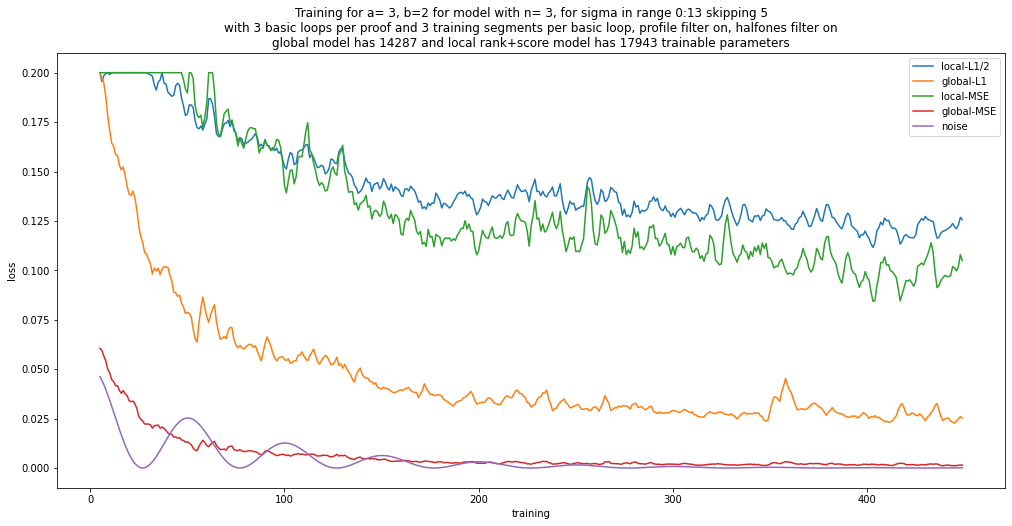}

\bigskip

The results are rather chaotic, with the node values getting close to or at the minimum value of $5$ but then
bouncing back considerably. In the middle of training the results were a little more consistent, then getting worse as training 
increased. One might conjecture that it got worse later due to the networks memorizing the answers for all the $\sigma$ values
except $5$, thereby degrading the performance at $\sigma = 5$.

\subsection{Other sizes}
\label{graphOther}

We include here a few graphs of node numbers for various other sizes. See the captions of the
diagrams to describe the individual cases. The loss graphs aren't included, as they look fairly similar to the ones above. 

\bigskip

\includegraphics[scale=0.4]{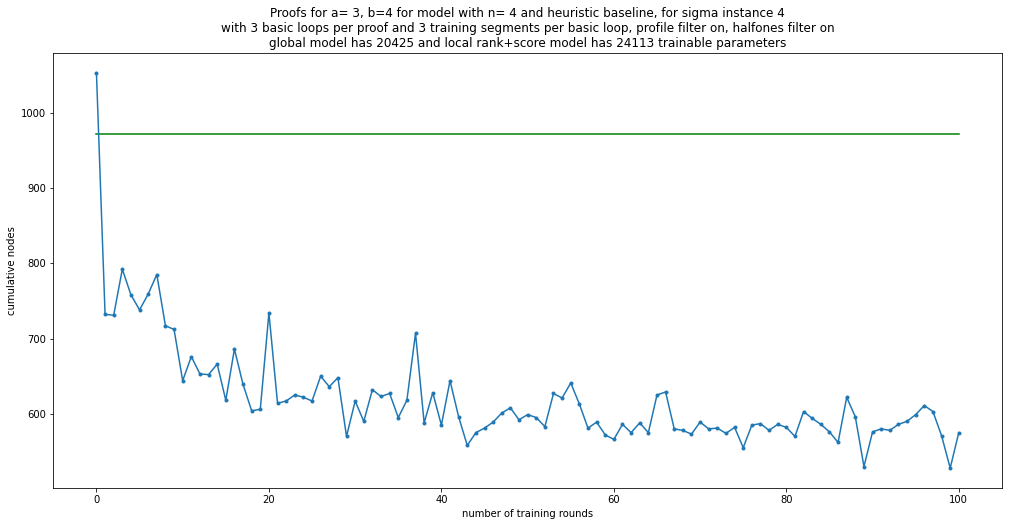}

\includegraphics[scale=0.4]{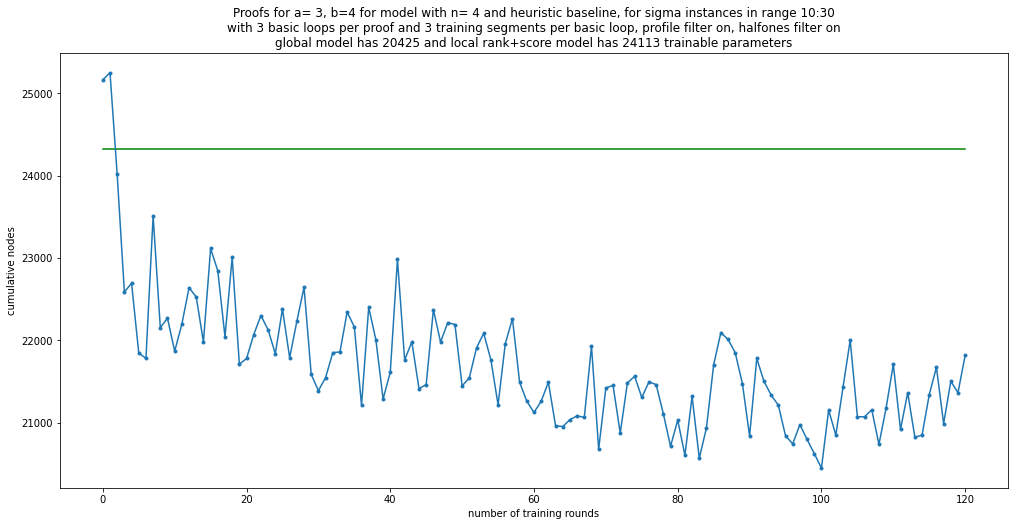}

\includegraphics[scale=0.4]{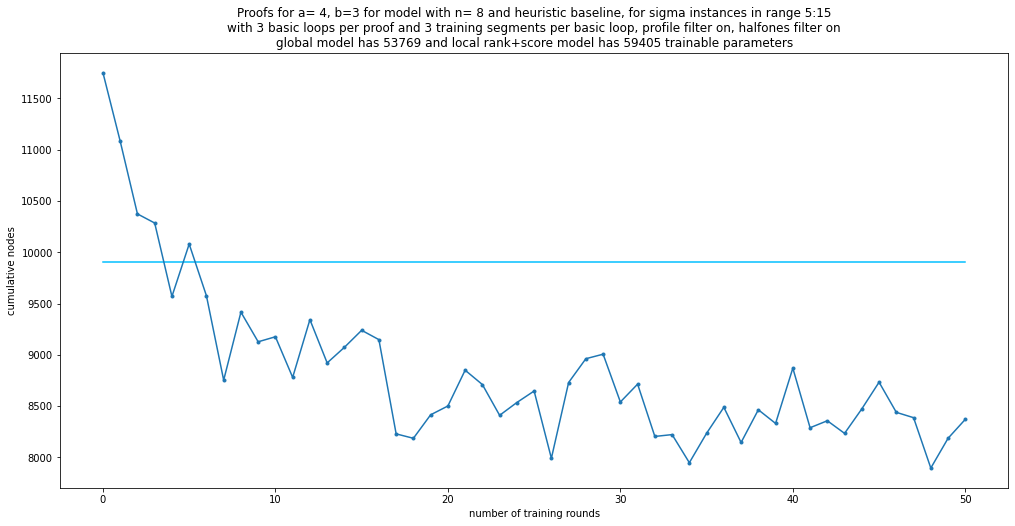}

\includegraphics[scale=0.4]{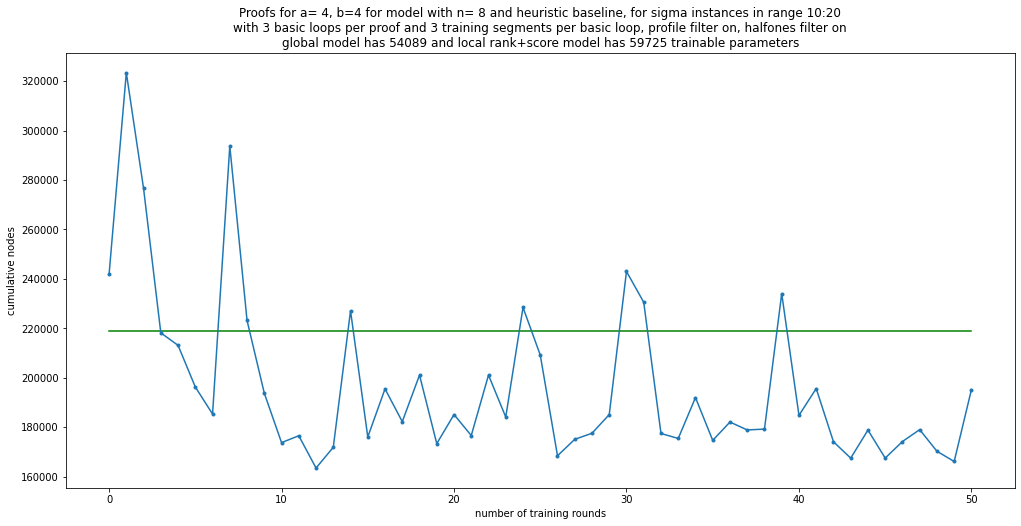}

\bigskip

In the last case, concerning semigroups of size $10$ with a segment of instances $\sigma$ of length $10$, 
the total number of parameters of both networks is 113814, whereas the number of nodes in the 
proof is more than 160000.

\section{Proof of minimality}
\label{minimality}

We now go back to smaller values
of $(a,b)$, namely let's look at the case $a=3$ and $b=2$.  There are 13 instances of $\sigma$ for initiating the proof.
If we choose one of the ones with a larger proof size, namely $\sigma = 3$, the training process seems to lead
to a minimum of  ${\rm nodes} = 37$ as we saw in \ref{graph32} above. It is natural to conjecture that this is in fact the
theoretical minimum for the number of nodes in the proof.

In view of the small size of this example it was feasible to find the minimal size and show that
it is indeed $37$. We calculate the theoretical minima for all cases of size $(a,b)=(3,2)$. 
Note that these proofs use the profile filter and half-ones filter (see \ref{addfilt}). For certain values 
such as $\sigma = 3$ it wasn't possible to find
the minimum without the filters, since the depth becomes too big, for example 
the benchmark number of nodes for $\sigma = 3$ without additional
filters is $537$.

\begin{theorem}
\label{theormin}
For the case $a=3$ and $b=2$, the minimal number of nodes $\nu_{\rm min}$ in a classification proof according to our scheme
is given in the following table:

\bigskip

\noindent
\hspace*{2cm}
\begin{tabular}{|l|c|c|c|c|c|c|c|c|c|c|c|c|c|}
\hline
$\sigma$            &  $0$ &  $1$ &  $2$ &  $3$ &  $4$ &  $5$ &  $6$ &  $7$ &  $8$ &  $9$ & $10$ & $11$ & $12$ \\
\hline
$\nu_{\rm min}$ &  $9$ & $21$& $23$& $37$& $11$&  $5$ & $3$  & $5$  &  $3$ & $11$& $3$  & $3$   & $17$ \\
\hline
\end{tabular}

\bigskip

\noindent
Our pair of neural networks configured and trained as described in previous sections is able to find
a minimal proof in each case. The minimum for all $\sigma$ instances together is $151$, and the model is 
able to find this value, although sparsely (see the graphics in \ref{graph32} above). 
\end{theorem}
\begin{proof} {\em [Indication]}
We calculate the minimal size of proof $\nu_{\rm min}$ 
in the following way. We successively create a tree-like object where each vertex corresponds to
the result of a succession of cuts with its associated mask. 
Below a vertex $v$ are new vertices corresponding to each of the locations $(x,y,p)$ that are available
in the mask associated to $v$, and at which we place the masks resulting from cutting at $(x,y,p)$ then processing. 
Vertices corresponding to done or impossible masks are not included. As this object is being created, we also run the proof 
model on each new set of vertices. This gives an upper bound for the number of nodes below a given one in the proof. 
The lower bound on new vertices is set to $1$. The upper and lower bounds are then propagated upwards in the tree,
by the rule that the number of nodes associated to the cut $(x,y)$ at a vertex $v$, is the sum over $p$ of the number of
nodes at each available $(x,y,p)$. Then, the number of nodes associated to $v$ is equal to the minimum 
of these values over available
$(x,y)$, plus $1$ for $v$ itself. This propagation is the same for the lower and upper bounds. 

The tree is furthermore pruned at each successive step by the following rules: if a vertex gets a lower bound equal to its
upper bound then it is removed from play. Also, if a collection of vertices associated to $(x,y)$ yield a lower bound that
is greater than the minima of the upper bounds over all $(x,y)$ under a given vertex, then that collection of vertices will 
not yield anything useful and they are pruned. 

We note that the pruning is essential---otherwise the size of the tree needed to calculate the minimum would be way too big. 

In the pruning process, a very small improvement may be seen by using our proving scheme and trained model, as opposed to using the
benchmark heuristic strategy, to calculate the upper bounds. It is not essential to use this improvement, though. 

The tree extension, proof computation, propagation and pruning steps are repeated until the lower bound and upper bound
at the root vertex coincide, this is then the minimal value. For runtime reasons, we calculated separately the 
bounds for the $27$ vertices obtained after the first cut.

\bigskip

\noindent
{\em Warning:} my implementation of the above strategy of proof 
is not certified to be correct, so this should only be considered as an indication of
proof. These minimal values do agree with the smallest values found by the neural networks, so it seems likely that they are correct.  

\bigskip

To illustrate the procedure, 
and also to highlight what the the neural networks need to do to find a minimal proof,
here are the matrices giving numbers of nodes depending on the initial cut location $(x,y)$ at the root. The minimal value
for the instance $\sigma$ is then $\nu _{\rm min} = k+1$ where $k$ is the minimal value of the entries in the matrix. 
The $+1$ is for the root node
itself. 
$$
\sigma = 0 \;\;\; 
\left(
\begin{array}{ccc}
8 & 8  & 8 \\
8  & 8  & 8 \\
8  & 8  & 8 \\
\end{array}
\right)
\;\;\; \nu _ {\rm min} = 9
\qquad \qquad \qquad
\sigma = 1 \;\;\; 
\left(
\begin{array}{ccc}
34 & 26  & 26 \\
20 & 20 & 20 \\
20 & 20 & 20 \\
\end{array}
\right)
\;\;\;\nu _ {\rm min}= 21
$$
$$
\sigma = 2 \;\;\; 
\left(
\begin{array}{ccc}
30 & 29 & 28 \\
29 & 30 & 28 \\
22 & 22 & 22
\end{array}
\right)
\;\;\; \nu _ {\rm min} = 23
\qquad \qquad \qquad
\sigma = 3 \;\;\; 
\left(
\begin{array}{ccc}
36 & 36 & 36\\
36 & 36 & 36 \\
36 & 36 & 36
\end{array}
\right)
\;\;\; \nu _ {\rm min} = 37
$$
$$
\sigma = 4 \;\;\; 
\left(
\begin{array}{ccc}
18 & 13 & 13 \\
12 & 12 & 10 \\
12 & 10 & 12 \\
 \end{array}
\right)
\;\;\; \nu _ {\rm min} = 11
\qquad \qquad \qquad
\sigma = 5 \;\;\; 
\left(
\begin{array}{ccc}
4 & 6 & 7 \\
6 & 4 & 7 \\
5 & 5 & 5 \\
\end{array}
\right)
\;\;\; \nu _ {\rm min}= 5
$$
$$
\sigma = 6 \;\;\; 
\left(
\begin{array}{ccc}
2 & 3 & 5 \\
3 & 2 & 5 \\
3 & 3 & 3 \\
\end{array}
\right)
\;\;\; \nu _ {\rm min} = 3
\qquad \qquad \qquad
\sigma = 7 \;\;\; 
\left(
\begin{array}{ccc}
4 & 6 & 6 \\
6 & 4 & 5  \\
6 & 5 & 4
\end{array}
\right)
\;\;\; \nu _ {\rm min} = 5
$$
$$
\sigma = 8 \;\;\; 
\left(
\begin{array}{ccc}
2 & 3 & 3 \\
3 & 2 & 3 \\
3 & 3 & 2
\end{array}
\right)
\;\;\; \nu _ {\rm min}= 3
\qquad \qquad \qquad
\sigma = 9 \;\;\; 
\left(
\begin{array}{ccc}
14 & 14 & 13 \\
14 & 14 & 13\\
10 & 10 & 10
\end{array}
\right)
\;\;\; \nu _ {\rm min}= 11
$$
$$
\sigma = 10 \;\;\; 
\left(
\begin{array}{ccc}
2 & 3 & 3 \\
3 & 2 & 4\\
3 & 4 & 2
\end{array}
\right)
\;\;\; \nu _ {\rm min}= 3
\qquad \qquad \qquad
\sigma = 11 \;\;\; 
\left(
\begin{array}{ccc}
2 & 2 & 3 \\
2 & 2 & 3 \\
3 & 3 & 2
\end{array}
\right)
\;\;\; \nu _ {\rm min}= 3
$$
$$
\sigma = 12 \;\;\; 
\left(
\begin{array}{ccc}
16 & 16 & 16 \\
16 & 16 & 16 \\
16 & 16 & 16
\end{array}
\right)
\;\;\;\nu _ {\rm min} = 17
$$

Note that for $\sigma$ instances $0$, $3$ and $12$ the number of nodes doesn't depend on the first cut. 
Of course it depends
on subsequent cuts. 

Let us consider this question in further detail for the case $\sigma = 3$. 
Node counts depending on the first cut location are as follows:

\vspace*{0.4cm}

\hspace*{2cm}
\begin{tabular}{|c||c|c|c|}
\hline
& $y=0$ & $y=1$ & $y=2$ \\
\hline 
\hline
$x=0$ & $13 + 10 + 13$ & $14 + 10 + 12$ & $14 + 10 + 12$  \\
\hline
$x=1$ & $14 + 10 + 12$ & $13 + 10 + 13$  & $14 + 10 + 12$  \\
\hline
$x=2$ & $14 + 10 + 12$ & $14 + 10 + 12$ & $13 + 10 + 13$   \\
\hline
\end{tabular}

\vspace*{0.4cm}

\noindent
The table entries refer to the number of nodes at the values $p=0$, $p=1$ and $p=2$. Thus, 
$14+10+12$ indicates $14$ nodes for $p=0$, $10$ nodes for $p=1$ and $12$ nodes for $p=2$.
The table says, for example, that the count of nodes at $(x,y,p)=(0,0,0)$ is $13$, whereas
the count for $(x,y,p)=(1,2,0)$ is $14$. 

As pointed out before, the sums are $36$ independently of $(x,y)$. 
Adding one for the root node gives the desired value of $37$ 
nodes for the full proof at $\sigma = 3$.

Continue by looking at the node values for the next choice of cuts in
a sample case. We'll consider a node $v$ obtained at location $(x,y,p)$ from the first choice of cut.
Recall that a choice of cut is a choice of $(x,y)$ yielding in this case three vertices below corresponding to
$(x,y,0)$,$(x,y,1)$ and $(x,y,2)$.

Here are some matrices $(n_{x',y'})$ that give the number of nodes (but not including $1$ for $v$,
that is added into $\nu _{\rm min}$)
corresponding to a second cut $(x',y')$. The values $(x',y')$ corresponding to the previous cut are 
naturally unavailable. 
$$
\bullet\;
\mbox{vertex from first cut at}\; (0,0,0),\; \mbox{lower bounds for next cut: }
\left[
\begin{array}{ccc}
- & 12 & 12 \\
12 & 13 & 14 \\
12 & 14 & 13 \\
\end{array}
\right]
\;\;\; \nu_{\rm min}= 13
$$
$$
\bullet\;
\mbox{vertex from first cut at}\; (0,0,1),\; \mbox{lower bounds for next cut: }
\left[
\begin{array}{ccc}
- & 9 & 9 \\
9 & 9 & 9 \\
9 & 9 & 9 
\end{array}
\right]
\;\;\; \nu_{\rm min} = 10
$$
$$
\bullet\;
\mbox{vertex from first cut at}\; (0,0,2),\; \mbox{lower bounds for next cut: }
\left[
\begin{array}{ccc}
- & 14 & 14 \\
12 & 13 & 15 \\
12 & 15 & 13
\end{array}
\right]
\;\;\; \nu_{\rm min}= 13
$$

These are going to enter into the full minimum value at the cut location $(x,y)=(0,0)$. In the first and third cases,
the neural network has to choose correctly the next cut in order to get a minimal value.

The program that does the minimality 
proof of Theorem \ref{theormin} will obtain the analogous information at all nodes of the possible proof trees that aren't
discarded as not being in the running for minimal ones. I don't have a good method for visualizing all the information.
This completes our summary of the computations that go into the minimality proof. 
\end{proof}

\section{Addendum: The process function}
\label{appendix}

We record here the pytorch functions written to go into the {\em process} function of Section \ref{task}, the function 
that implements logical
consequences of the associativity axiom on positions of a classification proof. 
As well as being things that one should verify, 
these programs serve to illustrate how to use boolean tensor manipulations under pytorch to replace 
\verb+for...next+ loops, a technique used systematically in order to improve the computation speed. 

A few minor modifications are made for readability. The function \verb+arangeic+ is the function \verb+arange+, that is
to say the sequence of consecutive integer values starting from $0$ of the given length, placed
onto the required device (CPU or GPU if available). We note that \verb+prod+, \verb+left+, \verb+right+ and \verb+ternary+ 
are the tensors denoted $m$, $l$, $r$ and $t$ in Section \ref{task},
as contained in the position represented by a dictionary \verb+Data+ with \verb+length+ denoting the batchsize. 
Dimension $0$ of all tensors is the batch dimension. 

\noindent
{\bf 1.} \, In the \verb+modifyternaryStep+ we insert $0$ into \verb+ternary+ at location $x,y,z,i$ (for $x,y,z\in A$ and $i\in I^0$) whenever,
for all $p\in B^0$ such that \verb+prod+$(x,y,p) = 1$ we have \verb+right+$(p,z,i)=0$ and similarly using \verb+left+.

{\footnotesize
\begin{verbatim}
    def modifyternaryStep(self,Data):
        a = self.alpha
        bz = self.beta + 1
        #
        length = Data['length']
        prod = Data['prod']
        left = Data['left']
        right = Data['right']
        ternary = Data['ternary']
        #
        ivx = arangeic(length).view(length,1,1,1,1).expand(length,a,a,a,bz)
        xvx = arangeic(a).view(1,a,1,1,1).expand(length,a,a,a,bz)
        yvx = arangeic(a).view(1,1,a,1,1).expand(length,a,a,a,bz)
        zvx = arangeic(a).view(1,1,1,a,1).expand(length,a,a,a,bz)
        pvx = arangeic(bz).view(1,1,1,1,bz).expand(length,a,a,a,bz)
        #
        nter0_left = (prod[ivx,yvx,zvx,pvx] & left[ivx,xvx,pvx,0]).any(4)
        nter1_left = (prod[ivx,yvx,zvx,pvx] & left[ivx,xvx,pvx,1]).any(4)
        #
        nter0_right = (prod[ivx,xvx,yvx,pvx] & right[ivx,pvx,zvx,0]).any(4)
        nter1_right = (prod[ivx,xvx,yvx,pvx] & right[ivx,pvx,zvx,1]).any(4)
        #
        nter0v = (nter0_left & nter0_right)
        nter1v = (nter1_left & nter1_right)
        #
        newternary = ternary.clone()
        newternary[:,:,:,:,0] = ternary[:,:,:,:,0] & nter0v
        newternary[:,:,:,:,1] = ternary[:,:,:,:,1] & nter1v
        #
        NewData = self.rr1.copydata(Data)
        NewData['ternary'] = newternary.detach()
        #
        return NewData
\end{verbatim} 
}

\noindent
{\bf 2.} \, 
In the \verb+modifyleftrightStep+ we insert $0$ into \verb+right+ at location $p,z,i$ (for $p\in B^0$, $z\in A$ and $i\in I^0$) whenever
there exist $x,y\in A$ such that the product $x\cdot y$ is uniquely defined equal to $p$ and 
\verb+ternary+$(x,y,z,i)=0$, and similarly for \verb+left+. 

{\footnotesize
[Additional linebreaks are inserted at the \verb+nleft0+, \verb+nleft1+, \verb+nright0+, \verb+nright1+ lines below so it fits on the page.]}

{\footnotesize
\begin{verbatim}
    def modifyleftrightStep(self,Data):
        a = self.alpha
        bz = self.beta + 1
        #
        length = Data['length']
        prod = Data['prod']
        left = Data['left']
        right = Data['right']
        ternary = Data['ternary']
        #
        prodstats = prod.to(torch.int64).sum(3)
        unique = (prodstats == 1)
        #
        ivx = arangeic(length).view(length,1,1,1,1).expand(length,a,a,a,bz)
        xvx = arangeic(a).view(1,a,1,1,1).expand(length,a,a,a,bz)
        yvx = arangeic(a).view(1,1,a,1,1).expand(length,a,a,a,bz)
        zvx = arangeic(a).view(1,1,1,a,1).expand(length,a,a,a,bz)
        pvx = arangeic(bz).view(1,1,1,1,bz).expand(length,a,a,a,bz)
        #
        nleft0 = (( (~prod[ivx,yvx,zvx,pvx]) |  
                            (~unique[ivx,yvx,zvx]) | ternary[ivx,xvx,yvx,zvx,0]).all(3)).all(2)
        nleft1 = (( (~prod[ivx,yvx,zvx,pvx]) |  
                            (~unique[ivx,yvx,zvx]) | ternary[ivx,xvx,yvx,zvx,1]).all(3)).all(2)
        #
        nright0 = (( (~prod[ivx,xvx,yvx,pvx]) |  
                            (~unique[ivx,xvx,yvx]) | ternary[ivx,xvx,yvx,zvx,0]).all(2)).all(1)
        nright1 = (( (~prod[ivx,xvx,yvx,pvx]) |  
                            (~unique[ivx,xvx,yvx]) | ternary[ivx,xvx,yvx,zvx,1]).all(2)).all(1)
        #
        newleft = left.clone()
        newright = right.clone()
        #
        newleft[:,:,:,0] = left[:,:,:,0] & nleft0
        newleft[:,:,:,1] = left[:,:,:,1] & nleft1
        newright[:,:,:,0] = right[:,:,:,0] & (nright0.permute(0,2,1))
        newright[:,:,:,1] = right[:,:,:,1] & (nright1.permute(0,2,1))
        #
        NewData = self.rr1.copydata(Data)
        NewData['left'] = newleft.detach()
        NewData['right'] = newright.detach()
        #
        return NewData
\end{verbatim}
}

\noindent
{\bf 3.} \, 
In the \verb+modifyprodStep+ we insert $0$ into \verb+prod+ at location $x,y,p$ (for $x,y\in A$ and $p\in B^0$) whenever,
there exists $z\in A$ and $i\in I^0 = \{ 0,1\}$, such that \verb+right+$(p,z,i) = 0$ and \verb+ternary+$(x,y,z,(1-i))=0$ 
(we note that if those exist and if $x\cdot y = p$ then $x\cdot y \cdot z$ can't be either $i$ or $(1-i)$, ruling out that possibility
so $x\cdot y \neq p$). Similarly for \verb+left+.

{\footnotesize 
\begin{verbatim}   
    def modifyprodStep(self,Data):
        a = self.alpha
        bz = self.beta + 1
        #
        length = Data['length']
        prod = Data['prod']
        left = Data['left']
        right = Data['right']
        ternary = Data['ternary']
        #
        lvx = arangeic(length).view(length,1,1,1,1).expand(length,a,a,a,bz)
        xvx = arangeic(a).view(1,a,1,1,1).expand(length,a,a,a,bz)
        yvx = arangeic(a).view(1,1,a,1,1).expand(length,a,a,a,bz)
        zvx = arangeic(a).view(1,1,1,a,1).expand(length,a,a,a,bz)
        pvx = arangeic(bz).view(1,1,1,1,bz).expand(length,a,a,a,bz)
        #
        leftbin01 = (left[lvx,xvx,pvx,0] | ternary[lvx,xvx,yvx,zvx,1]) 
        leftbin10 = (left[lvx,xvx,pvx,1] | ternary[lvx,xvx,yvx,zvx,0])
        #
        rightbin01 = (right[lvx,pvx,zvx,0] | ternary[lvx,xvx,yvx,zvx,1])
        rightbin10 = (right[lvx,pvx,zvx,1] | ternary[lvx,xvx,yvx,zvx,0])
        #
        newprod = prod.clone()
        newprod = newprod & ( (leftbin01 & leftbin10).all(1) )
        newprod = newprod & ( (rightbin01 & rightbin10).all(3) )
        #
        NewData = self.rr1.copydata(Data)
        NewData['prod'] = newprod.detach()
        #
        return NewData
\end{verbatim}
}

\noindent
{\bf 4.} \,        
These functions, serving to add some additional $0$'s to our tensors due to the associativity 
axiom, are put together in the \verb+process+ function. The text below contains some previously defined things with relatively 
self-explanatory names for which we refer to the program source.  Note that the process is repeated until no new $0$'s
are found (measured by \verb+knowledge+), and we do the repetition on subsets of the batch in order to save computation time (the batch size might
start out as $500$ but maybe only a few locations require multiple iterations of the process).

{\footnotesize
\begin{verbatim}        
    def process(self,Data):
        length = Data['length']
        if length == 0:
            return Data
        #
        OutputData = self.rr1.copydata(Data)
        nprod = Data['prod']
        nprodstats = nprod.to(torch.int64).sum(3)
        subset = ((nprodstats > 0).all(2)).all(1)
        NextData = self.rr1.detectsubdata(Data,subset)
        if subset.to(torch.int).sum(0) == 0:
            return OutputData
        for i in range(1000):
            priorknowledge = self.rr1.knowledge(NextData)
            #
            NextData = self.modifyternaryStep(NextData)
            #
            NextData = self.modifyleftrightStep(NextData)
            #
            NextData = self.modifyprodStep(NextData)
            #
            nextknowledge = self.rr1.knowledge(NextData)
            nextdonedetect = (priorknowledge >= nextknowledge)
            subset_nextdone = composedetections(length,subset,nextdonedetect)
            NextDoneData = self.rr1.detectsubdata(NextData,nextdonedetect)
            OutputData = self.rr1.insertdata(OutputData,subset_nextdone,NextDoneData)
            #
            subset = subset & (~subset_nextdone)
            if subset.to(torch.int).sum(0) == 0:
                break
            NextData = self.rr1.detectsubdata(NextData, ~nextdonedetect )
        return OutputData
\end{verbatim}
}

\noindent
The program source is available at 
\verb+https://github.com/carlostsimpson/sg-learn+
\newline
and might be bundled with this preprint.

\

\

\bigskip

\noindent
Carlos Simpson, {\sc CNRS, Universit\'e C\^ote d'Azur, LJAD}
\newline
carlos.simpson@univ-cotedazur.fr
\newline
Nice, France

\end{document}